\documentclass[twoside]{article}

\usepackage[accepted]{aistats2025}
\usepackage[round]{natbib}

\usepackage{amsmath, amsthm, amssymb, amsfonts,bm}
\usepackage{xcolor}

\newtheorem{theorem}{Theorem}
\newtheorem{proposition}{Proposition}
\newtheorem{lemma}{Lemma}

\newtheorem{definition}{Definition}

\newtheorem{remark}{Remark}

\newtheorem*{theorem*}{Theorem}
\newtheorem*{proposition*}{Proposition}
\newtheorem*{lemma*}{Lemma}
\newtheorem*{corollary*}{Corollary}
\newtheorem*{example*}{Example}

\usepackage{apptools} 
\AtAppendix{\counterwithin{lemma}{section}}
\AtAppendix{\counterwithin{theorem}{section}}
\AtAppendix{\counterwithin{proposition}{section}}
\AtAppendix{\counterwithin{corollary}{section}}
\AtAppendix{\counterwithin{example}{section}}
\AtAppendix{\counterwithin{definition}{section}}

\def\vzero{{\bm{0}}}
\def\0{\bm{0}}

\def\1{\bm{1}}

\def\va{{\bm{a}}}

\def\vc{{\bm{c}}}

\def\vu{{\bm{u}}}
\def\vv{{\bm{v}}}
\def\vw{{\bm{w}}}
\def\vx{{\bm{x}}}

\def\vU{{\bm{U}}}

\def\vW{{\bm{W}}}

\DeclareMathAlphabet{\mathsfit}{\encodingdefault}{\sfdefault}{m}{sl}
\SetMathAlphabet{\mathsfit}{bold}{\encodingdefault}{\sfdefault}{bx}{n}

\def\sN{{\mathbb{N}}}

\def\sR{{\mathbb{R}}}

\def\E{\displaystyle \mathbb{E}}
\def\Var{\displaystyle \mathrm{Var}}

\newcommand{\loss}{\mathcal{L}}
\newcommand\op[1]{\operatorname{#1}}
\newcommand{\norm}[1]{\left\lVert#1\right\rVert}

\DeclareMathOperator*{\argmax}{arg\,max}
\DeclareMathOperator*{\argmin}{arg\,min}

\newcommand{\eg}{\textit{e.g., }}
\newcommand{\ie}{\textit{i.e., }}

\usepackage{amsmath}
\usepackage{amssymb}
\usepackage{mathtools}
\usepackage{amsthm}

\usepackage{wrapfig}

\usepackage{booktabs}

\usepackage{hyperref}
\usepackage{url}

\newcommand\ours{\textsf{SSEM}}
\newcommand{\sumi}{\sum_{i\in[m]}}

\definecolor{mydarkblue}{rgb}{0,0.08,0.45}
\hypersetup{ %
pdftitle={},
pdfkeywords={},
pdfborder=0 0 0,
pdfpagemode=UseNone,
colorlinks=true,
linkcolor=mydarkblue,
citecolor=mydarkblue,
filecolor=mydarkblue,
urlcolor=mydarkblue,
}

\begin{document}


\runningtitle{A Theoretical Framework for Preventing Class Collapse in Supervised Contrastive Learning}


\runningauthor{Chungpa Lee, Jeongheon Oh, Kibok Lee,  Jy-yong Sohn}

\twocolumn[

\aistatstitle{A Theoretical Framework for Preventing Class Collapse \\ in Supervised Contrastive Learning}

\aistatsauthor{ Chungpa Lee \And Jeongheon Oh \And Kibok Lee \And  Jy-yong Sohn$^\dag$ }

\aistatsaddress{ Yonsei University \And  Bank of Korea \\ Yonsei University$^*$ \And Yonsei University \And Yonsei University }
]

\begin{abstract}
Supervised contrastive learning (SupCL) has emerged as a prominent approach in representation learning, leveraging both supervised and self-supervised losses. However, achieving an optimal balance between these losses is challenging; failing to do so can lead to class collapse, reducing discrimination among individual embeddings in the same class. In this paper, we present theoretically grounded guidelines for SupCL to prevent class collapse in learned representations. Specifically, we introduce the Simplex-to-Simplex Embedding Model (\ours{}), a theoretical framework that models various embedding structures, including all embeddings that minimize the supervised contrastive loss. Through \ours{}, we analyze how hyperparameters affect learned representations, offering practical guidelines for hyperparameter selection to mitigate the risk of class collapse. Our theoretical findings are supported by empirical results across synthetic and real-world datasets.
\end{abstract}

\vspace{8pt}

\section{INTRODUCTION}

Contrastive learning (CL) has recently demonstrated significant advancements in self-supervised representation learning, where data representations are learned by comparing views generated by augmentations of the training data~\citep{chen2020simple, he2020momentum, radford2021clip}. Specifically, CL maximizes the similarity between positive pairs, derived from different views of the same instance, while simultaneously minimizing the similarity between negative pairs, derived from views of distinct instances. While self-supervised CL does not leverage available supervision, such as class labels in classification tasks, incorporating such information can be advantageous for representation learning. To take advantage of this, \citet{khosla2020supervised} introduced supervised contrastive learning (SupCL), which extends self-supervised CL by treating views from different instances within the same class as positive pairs. Furthermore, \citet{islam2021broad} empirically analyzed the effectiveness of SupCL through experiments on the transferability of representations learned both with and without supervision.

However, recent studies have pointed out that optimizing the supervised contrastive loss often results in \emph{class collapse}, where the embedding vectors of all instances within the same class converge into the same point in the embedding space~\citep{graf2021dissecting, papyan2020prevalence}.
Class collapse significantly degrades the generalizability of learned representations by eliminating within-class variance that is potentially crucial for effective transfer learning \citep{islam2021broad, chen2022perfectly}. While several follow-up studies have explored the conditions under which class collapse occurs and proposed strategies to prevent it~\citep{feng2021rethinking, chen2022perfectly, wang2023opera, xue2023features}, their analyses are often limited to specific conditions, such as assumptions about data distribution. The precise mechanism underlying class collapse in SupCL is not yet fully understood.

To this end, we provide a theoretical analysis of the behavior of embeddings that minimize the supervised contrastive loss (SupCL loss), offering guidelines on how to avoid class collapse. Our analysis applies to a broad range of data configurations, including varying numbers of classes, instances, and augmentations, whenever the SupCL loss is given as a convex combination of the supervised and self-supervised contrastive losses.
Our contributions are summarized as follows:
\begin{itemize}
    \item
    In Sec.~\ref{sec:optimal:structure}, we propose the Simplex-to-Simplex Embedding Model (\ours{}), a theoretical framework for modeling diverse embedding structures. We prove that embeddings $\vU^\star$ that minimize the SupCL loss can only be found within \ours{}. This result establishes \ours{} as a fundamental tool for analyzing embeddings learned in SupCL.

    \item
    In Sec.~\ref{sec:emb:var}, we provide guidelines for designing the SupCL loss to mitigate class collapse in learned embeddings. In particular, we derive a mathematical expression for the variance of embeddings $\vU^\star$, a key metric for assessing class collapse when the within-class variance is zero. Furthermore, we characterize the relationship between embedding variances and the hyperparameters.

    \item
    In Sec.~\ref{sec:experiment}, we present experiments on both synthetic and real-world datasets, demonstrating that our theoretical findings hold in practice. Specifically, the variance of the learned embeddings aligns with our theoretical predictions, allowing us to identify optimal hyperparameters in SupCL that balance within-class and between-class variance, leading to improved transfer learning performance.
\end{itemize}

\section{RELATED WORK}

\paragraph{SupCL Methods.}
SupCL leverages supervised information by jointly utilizing the supervised contrastive loss and the InfoNCE-based self-supervised loss \citep{khosla2020supervised}, demonstrating superior performance compared to existing CL methods that only use self-supervised loss \citep{islam2021broad, gunel2020supervised}. To further enhance SupCL, \citet{chen2022perfectly} replace the InfoNCE-based self-supervised loss with the class-conditional InfoNCE loss. Meanwhile, \citet{feng2021rethinking} refine the supervised contrastive loss by selecting only the $k$ nearest neighbors within each class as positive pairs, while \citet{wang2023opera} introduce hierarchical supervision to balance instance-level and class-level information. Building on these approaches, \citet{oh2024effectiveness} incorporate supervision into asymmetric non-CL methods \citep{grill2020bootstrap, chen2021exploring} to further improve representation learning.

Despite these advancements, most existing methods require extensive hyperparameter tuning to achieve optimal performance. Rather than proposing a new SupCL method, this paper conducts a theoretical analysis on the role of hyperparameters in the SupCL loss function, ensuring that searching within a narrower region is sufficient and thus enabling more efficient hyperparameter optimization.

\paragraph{Analysis on SupCL.}

Although considerable theoretical research has been conducted on CL \citep{arora2019theoretical, parulekar2023infonce, wen2021toward, yang2023understanding}, the understanding of SupCL remains underexplored. Notably, \citet{xue2023features} showed that the bias of gradient descent can cause subclass representation collapse and suppress harder class-relevant features. However, their analysis is specific to spectral contrastive loss \citep{haochen2021provable}, which is rarely used in practice. Moreover, most existing studies rely on strict assumptions about data distribution, limiting their applicability to real-world datasets. Instead of focusing on a specialized loss function or imposing such assumptions, we directly optimize the widely adopted SupCL loss \citep{khosla2020supervised, islam2021broad, oh2024effectiveness}, ensuring broader applicability across diverse datasets.

\paragraph{Understanding SupCL Through Embedding Structures.}

Several previous works aim at understanding the optimal embedding structures minimizing the supervised loss and/or contrastive loss. In the supervised learning setup, one well-known phenomenon of optimal embeddings is neural collapse \citep{papyan2020prevalence}, where the embeddings collapse to a simplex Equiangular Tight Frame (ETF). Similarly, in the case of CL, \citet{lu2022neural} showed that the optimal embeddings minimizing the softmax-based contrastive loss construct the simplex ETF, and \citet{lee2024analysis} generalized the result to the cases of minimizing other CL losses including the sigmoid-based loss \citep{zhai2023sigmoid}.
For SupCL, the optimal embeddings that minimize the supervised contrastive loss result in class collapse, where embeddings of the same class collapse to a single point of simplex ETF \citep{graf2021dissecting}.

Notably, \citet{chen2022perfectly} introduced the class-conditional InfoNCE loss~\citep{oord2018representation} to spread out the embedding vectors within each class, and combined it with the supervised contrastive loss. They demonstrated that non-collapsed embeddings can achieve the lower value of the combined loss than collapsed ones in the specific case where the number of classes is two or three, without identifying the optimal embeddings. Although \citet{chen2022perfectly} provided a meaningful direction of designing the SupCL loss in a way that the optimal embeddings do not suffer from class collapse, this work cannot be extended to general cases when the number of classes is more than three. In addition, none of existing works examined the behavior of the optimal embeddings in the SupCL setup, specifically regarding how to construct a loss function that consistently avoids class collapse.

\vspace{16pt}
Building on prior works, this paper specifies the optimal embeddings that minimize a convex combination of supervised and self-supervised losses, providing guidelines for selecting hyperparameters to avoid class collapse in SupCL.

\section{PROBLEM FORMULATION}
\label{sec:formulation}

We consider the problem of training an encoder $f$ that maps the feature $\vx$ into the embedding $\vu = f(\vx)$ by using SupCL. The training sample is categorized into $m$ classes, and the sample size for each class is $n$. Every instance is augmented, \ie generating similar instances by data augmentation techniques; the number of augmentation for each instance is denoted by $p$. We use the notation $\vx_{i,j,k}$ to represent the feature of $k$-th augmentation of $j$-th instance in $i$-th class for $i\in [m], j \in [n]$ and $k \in [p]$, where we define $[m]:=\{1, 2, \cdots, m\}$ for positive integer $m$.

The output of the encoder $f$, also referred to \textit{embedding}, is denoted as $\vu_{i,j,k} = f(\vx_{i,j,k}) \in \sR^d$, where $d$ is the embedding dimension. To streamline the notation, we define several sets of embedding vectors:
\begin{itemize}
    \item \textit{Same-instance embedding set}: \\ This is the set of embeddings for the $j$-th instance in $i$-th class, denoted by $\vU_{i,j} = \{\vu_{i,j,k} \}_{k \in[p]}$ for all $i\in[m]$ and $j\in[n]$.
    \item \textit{Same-class embedding set}: \\ This set contains the embeddings of all instances in the $i$-th class, denoted by $\vU_{i} = \cup_{j\in[n]} \vU_{i,j}$ for all $i\in[m]$.
    \item \textit{Entire embedding set}: \\ This set includes the embeddings of all instances, represented as $\vU = \cup_{i\in[m]} \vU_{i}$.
\end{itemize}

Note that the number of embeddings in each set is $|\vU|=mnp$ and $|\vU_i|=np$ for all $i\in[m]$. Throughout the paper, we assume that the encoder is normalized, \ie $\lVert f(\vx) \rVert_2 = 1$ for all input $\vx$, which is widely used in related works \citep{wang2017normface, wu2018unsupervised, tian2020contrastive, wang2020understanding, zimmermann2021contrastive, sreenivasan2023mini, lee2024analysis}.

The encoder is trained by optimizing the SupCL loss denoted as 
\begin{equation}
    \label{eq:loss}
    \loss (\vU):=(1-\alpha)\; \loss_{\op{Sup}}(\vU)+\alpha\; \loss_{\op{Self}}(\vU),
\end{equation}
where $\loss_{\op{Sup}}(\vU)$ is the supervised contrastive loss that considers the supervision (class information of each instance), $\loss_{\op{Self}}(\vU)$ is the self-supervised contrastive loss that does not make use of the class information, and $\alpha \in [0,1]$ is the coefficient for combining two losses. 
Here, each loss term is defined as 
\begin{align}
\loss_{\op{Sup}}(\vU) 
& = \label{eq:loss-sup}
- \frac{1}{mn(n-1)p^2}
\\& \nonumber
\sum_{\substack{i\in[m] \\ j\ne j'\in[n]}}
\sum_{\vu \in \vU_{i,j}}\sum_{\vv \in \vU_{i,j'}}
\log \frac{\exp(\vu^\top \vv / \tau)}{\sum_{\vw \in \vU}\exp(\vu^\top \vw / \tau)}
\end{align}
and
\begin{align}
\loss_{\op{Self}}(\vU)
& = \label{eq:loss-cl}
-\frac{1}{mnp^2}
\\& \nonumber
\sum_{\substack{i\in[m] \\ j\in[n]}}
\sum_{\vu \in \vU_{i,j}}\sum_{\vv \in \vU_{i,j}}
\log \frac{\exp(\vu^\top \vv / \tau)}{ \sum_{\vw \in \vU}\exp(\vu^\top \vw / \tau)},
\end{align}
where $j \ne j' \in [n]$ is the simplified notation representing $j\in[n]$ and $j'\in[n] \setminus \{j\}$. 
Note that $\loss_{\op{Sup}}(\vU)$ in \eqref{eq:loss-sup} is slightly different from what was proposed in the original supervised contrastive learning paper \citep{khosla2020supervised} which does not include the condition $j\ne j'$ in the first summation of \eqref{eq:loss-sup}.
We add this condition to make sure that the positive pairs ($\vu$ and $\vv$) counted in $\loss_{\op{Sup}}(\vU)$ and $\loss_{\op{Self}}(\vU)$ do not overlap; the augmented entities $\vu, \vv \in \vU_{i,j}$ of the same instance are counted in  $\loss_{\op{Self}}(\vU)$, while the augmented entities ($\vu \in \vU_{i,j}$ and $\vv \in \vU_{i,j'}$) of different instances in the same class are counted in  $\loss_{\op{Sup}}(\vU)$.

To make our key findings easier to understand, we first consider the simplest case where $p=1$, \ie 
each instance has only one augmentation.
In such a case, the index $k$ for the augmentation (in the embedding vector $\vu_{i,j,k}$) disappears, and the loss terms in \eqref{eq:loss-sup} and \eqref{eq:loss-cl} reduce to 
{%
\thinmuskip=1mu 
\medmuskip=1mu plus 2mu minus 4mu 
\thickmuskip=1mu plus 5mu
\begin{align}
\loss_{\op{Sup}}(\vU) 
&  = \label{eq:loss:sup:assume}
-\frac{1}{mn(n-1)}
\sum_{\substack{i\in[m] \\ j\ne j'\in[n]}}
\log \frac{\exp(\vu_{i,j}^\top \vu_{i,j'} / \tau)}{\sum_{\vw \in \vU}\exp(\vu_{i,j}^\top \vw / \tau)}
\end{align}
and
\begin{align}
\loss_{\op{Self}}(\vU)
&  = \label{eq:loss:cl:assume}
-\frac{1}{mn}
\sum_{\substack{i\in[m] \\ j\in[n]}}
\log \frac{\exp(\vu_{i,j}^\top \vu_{i,j} / \tau)}{\sum_{\vw\in\vU}\exp(\vu_{i,j}^\top \vw / \tau)}.
\end{align}
}

\begin{figure*}[t]
    \centering
    \begin{tabular}{ccc}
    \includegraphics[width=0.30\textwidth]{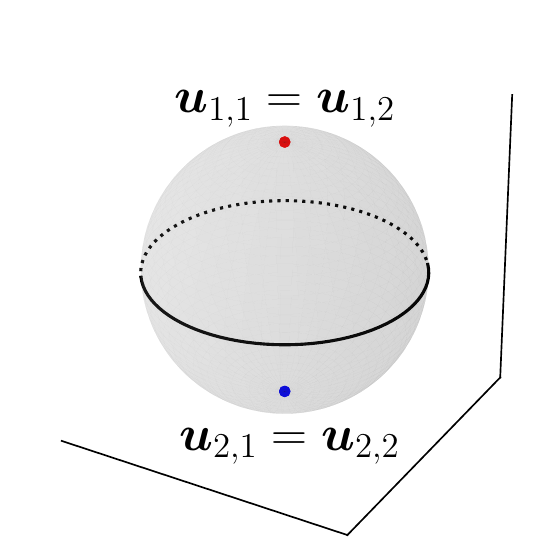}
    &\includegraphics[width=0.30\textwidth]{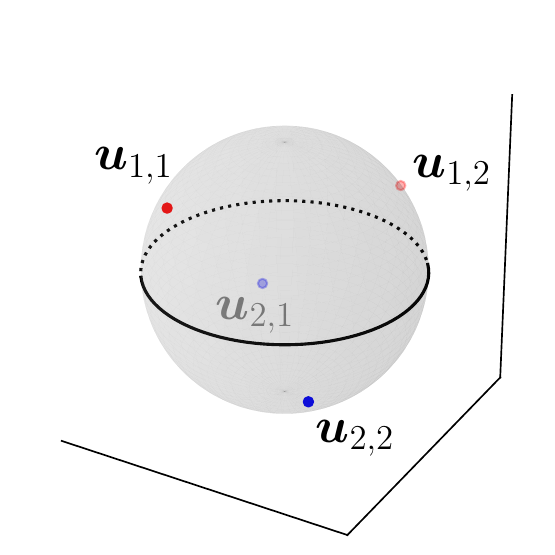}
    &\includegraphics[width=0.33\textwidth]{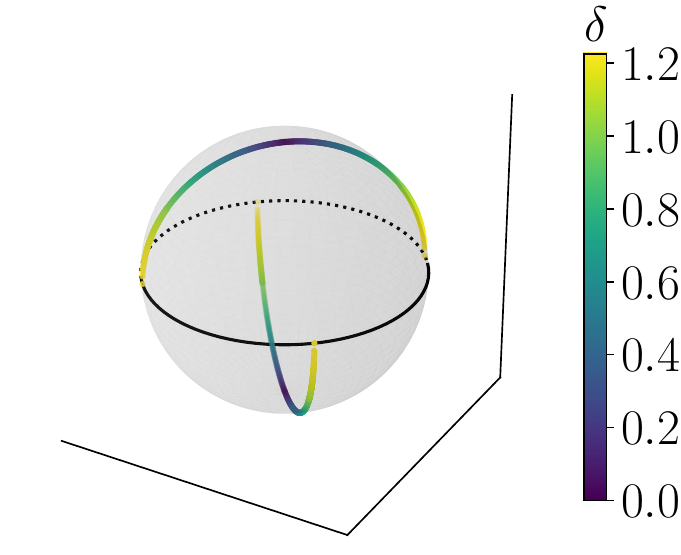} \\
    (a) $\delta=0$ & (b) $\delta=1$ & (c) $\delta\in \big[0, \sqrt{1.5}\big]$ \\
    \end{tabular}
    \caption{
    Illustration of the proposed Simplex-to-Simplex Embedding Model (SSEM) in Def.~\ref{def:model}, where both the number of classes ($m$) and the number of instances per class ($n$) are set to 2. The set of embedding vectors in \ours{} is denoted by $\vU = \{ \vu_{1,1}, \vu_{1,2}, \vu_{2,1}, \vu_{2,2} \}$, where the superscript $\delta$ in~\eqref{eq:model:entire:set} is omitted for simplicity. Each embedding's first subscript index indicates its class, with embeddings of class 1 drawn in red and those of class 2 in blue. The embeddings are visualized for different values of $\delta$:
    (a) When $\delta=0$, \ours{} is equal to $1$-simplex ETF, which is when class collapse happens. (b) When $\delta=1$, \ours{} is equal to $3$-simplex ETF, where every embedding is equidistant. (c) When $\delta$ varies in the range of $\big[0, \sqrt{1.5}\big]$, we visualize the trajectory of $\vu_{1,1}$ and $\vu_{1,2}$ in the upper arc, and the trajectory of $\vu_{2,1}$ and $\vu_{2,2}$ in the lower arc, where the color in the trajectory transits from purple to yellow as $\delta$ increases.
    }
    \label{fig:ssem}
\end{figure*}

Under the above setting, we focus on understanding the optimal embedding set 
\begin{equation}
    \label{eq:optimal:embedding:general}
    \vU^\star := \argmin_{\vU} \loss (\vU)
\end{equation}
that minimizes the SupCL loss in \eqref{eq:loss}. 
In Sec.~\ref{sec:optimal:structure} and Sec.~\ref{sec:emb:var}, we provide our theoretical results on the optimal embedding set $\vU^{\star}$ when $p=1$. These results are extended to general $p>1$ case in Appendix~\ref{appendix:proof}.

\section{OPTIMAL EMBEDDING 
}\label{sec:optimal:structure}

In this section, we first define Simplex-to-Simplex Embedding Model (\ours{}), a framework of embedding sets that models different types of geometric embedding vectors. 
Then, we show that the optimal embedding set $\vU^{\star}$ that minimizes the SupCL loss is only included in \ours{}; note that this result is helpful for analyzing the properties of optimal embeddings in the following sections.

\subsection{Simplex-to-Simplex Embedding Model}

Before defining our proposed \ours{}, we recall an embedding set called 
simplex equi-angular tight frame (ETF), where each vector is equally spaced from every other vector:
\begin{definition}[Simplex ETF]
\label{def:simplex}
A set of $n$ vectors $\vU$ on the $d$-dimensional unit sphere is called 
$(n-1)$-simplex ETF, if
\begin{align*}
&\lVert\vu\rVert_2^2=1 
\; \text{and} \;
\vu^\top\vv=-\frac{1}{n-1}, &\forall \vu \in \vU,  \vv \in \vU\setminus\{\vu\} 
.
\end{align*}
Note that $(n-1)$-simplex ETF exists when $d \ge n-1$. 
\end{definition}

Recall that our goal is to find the optimal embedding set $\vU^{\star}$ that minimizes the SupCL loss 
\begin{align*}
\loss(\vU) =  (1-\alpha)\; \loss_{\op{Sup}}(\vU)  + \alpha\; \loss_{\op{Self}}(\vU)
\end{align*}
in~\eqref{eq:loss}, which is a convex combination of $\loss_{\op{Self}}(\vU)$ and $\loss_{\op{Sup}}(\vU)$. According to recent works, the optimal embedding set that minimizes $\loss_{\op{Self}}(\vU)$ follows the $(mn-1)$-simplex ETF \citep{lu2022neural, lee2024analysis}, while the optimal embedding set that minimizes $\loss_{\op{Sup}}(\vU)$ follows the $(m-1)$-simplex ETF~\citep{graf2021dissecting}.
This means that we already know the solution $\vU^{\star}$ when $\alpha=0$ or $\alpha=1$, but not for the case of $0 < \alpha < 1$.
To find the solution $\vU^{\star}$ for all $\alpha$ values, we propose \ours{}, a framework that models the \textit{transition} from $(nm-1)$-simplex ETF to $(m-1)$-simplex ETF.
To be specific, \ours{} uses a single parameter $\delta$ that effectively controls the shift between two embedding sets, \ie $(nm-1)$-simplex ETF  and $(m-1)$-simplex ETF. 
Note that \ours{} stands for ``Simplex-to-Simplex Embedding Model'', due to its ability to explain the transition from one simplex ETF to another simplex ETF by changing $\delta$.
We formally define \ours{} as below:

\begin{definition}[Simplex-to-Simplex Embedding Model]
\label{def:model}
Let positive integers $m$,$n$, and a real value $\delta \in \Big[0,\sqrt\frac{mn-1}{m(n-1)}\Big]$ be given.  
We define Simplex-to-Simplex Embedding Model, denoted by ($m,n,\delta$)-\ours{}, as the set of $mn$ vectors
\begin{align}
    \label{eq:model:entire:set}
    \vU^\delta = 
    \big\{\vu_{i,j}^{\delta}\big\}_{i\in[m], j \in [n]}
\end{align}
satisfying the following:

For all $i \ne i'\in [m]$ and $j\ne j'\in[n]$,
\begin{align}
    \lVert\vu^\delta\rVert_2^2
    &\label{eq:model:same}
    = 1 
    \qquad\qquad\qquad\qquad\qquad
    \forall \vu^\delta \in \vU_{i,j}^\delta,
\end{align}
\begin{align}
    (\vu^\delta)^\top \vv^\delta
    &\label{eq:model:same:class}
    = 1 - \delta^2 \frac{mn}{mn-1}  
    \;
    \forall \vu^\delta \in \vU_{i,j}^\delta, \vv^\delta \in \vU_{i,j'}^\delta,
    \\
    (\vu^\delta)^\top \vv^\delta
    &\nonumber= 
    - \frac{1}{m-1} + \delta^2 \frac{m(n-1)}{(m-1)(mn-1)}
    \\&\label{eq:model:different} 
    \qquad\qquad\qquad\qquad
    \forall \vu^\delta \in \vU^{\delta}_i, \vv^\delta \in \vU_{i'}^\delta,
\end{align}
where $\vU_{i,j}^\delta := \{\vu_{i,j}^\delta \}$ and  $\vU^{\delta}_i: = \cup_{j\in[n]} \vU_{i,j}^\delta$.
\end{definition}

Let $\vu^\delta$ and $\vv^\delta$ be two distinct embedding vectors chosen from the set of $mn$ vectors forming ($m,n,\delta$)-\ours{} in Def.~\ref{def:model}. If $\vu^\delta$ and $\vv^\delta$ are from different instances ($j, j'$) of the same class $i$, their cosine similarity decreases as the parameter $\delta$ increases, as shown in \eqref{eq:model:same:class}. On the other hand, if $\vu^\delta$ and $\vv^\delta$ are from instances of different classes $i, i'$, their cosine similarity increases as $\delta$ gets larger, as stated in \eqref{eq:model:different}. Every embedding in ($m,n,\delta$)-\ours{} has unit norm, as described in \eqref{eq:model:same}.
As an example, Fig.~\ref{fig:ssem} visualizes the $(m,n,\delta)-$\ours{} for $m=n=2$ and various $\delta$ values, where the embedding dimension is set to $d=3$.

The below proposition shows the existence of \ours{} when the embedding dimension is sufficiently large, the proof of which is in Appendix~\ref{appendix:properties:ssem}.

\begin{proposition}[Existence of \ours{}]
    \label{thm:exist:model}
    Suppose $mn\geq 2$ and $d\geq mn-1$ hold.
Let a set of $mn$ vectors $\{\vw_{i,j}\}_{i\in[m], j\in[n]}$ forms the $(mn-1)$-simplex ETF in $\mathbb{R}^{d}$. For a given $\delta \in \Big[0,\sqrt\frac{mn-1}{m(n-1)}\Big]$,
define the set of $mn$ vectors $\vU^\delta := \big\{\vu_{i,j}^{\delta}\big\}_{i\in[m], j \in [n]}$ as
\begin{align}
    &\nonumber
    \vu_{i,j}^{\delta} := \delta \vw_{i,j} + h(\delta) \sum_{j'\in[n]} \vw_{i,j'}    \in \mathbb{R}^{d}
    \quad\forall i\in[m], j \in [n],
\end{align}
where
\begin{align}
    \nonumber
    h(\delta) := -\frac{\delta}{n} \pm \frac{1}{n}\sqrt{\frac{\delta^2m(1-n) +(mn-1)}{m-1}}.
\end{align}
Then, the set of $mn$ vectors $\vU^\delta$ constructs ($m,n,\delta$)-\ours{}.
\end{proposition}

The dimensionality assumption $d\geq mn-1$ for \ours{} in Proposition~\ref{thm:exist:model} is necessary because the simplex ETF defined in Def.~\ref{def:simplex} exists only when the dimension is sufficiently large.
Investigating the existence of \ours{} for $d<mn-1$ is closely related to the Thomson problem \citep{thomson1904xxiv}, which has been extensively studied in prior work \citep{sustik2007existence,fickus2012steiner, fickus2015tables, azarija2018there}. This challenge aligns with a common assumption in theoretical studies, where embeddings learned through supervised learning \citep{graf2021dissecting} and self-supervised learning \citep{lu2022neural, lee2024analysis} are typically analyzed under the condition that the dimension is sufficiently large. While our mathematical results in the following sections hold strictly for $d\geq mn-1$, our experimental findings in Sec.~\ref{sec:experiment} indicate that the properties of the optimal embedding discussed in Sec.~\ref{sec:emb:var} remain valid even when $d < mn-1$.

\subsection{Optimal Embeddings are in \ours{}}

Now we show that the optimal embedding set $\vU^\star$ in \eqref{eq:optimal:embedding:general} that minimizes the SupCL loss is only included in \ours{}, the proof of which is in Appendix~\ref{appendix:proof:ssem}.

\begin{theorem}[Optimality of \ours{}]
    \label{thm:optimal:embedding}
    Suppose $mn\geq 2$ and $d\geq mn-1$ hold.
    Then, all embedding sets $\vU^{\star}$ that minimize the loss $\loss(\vU)$ in \eqref{eq:loss} are included in the \ours{} in Def.~\ref{def:model}, \ie
\begin{align}
    \label{eq:main:theorem:existence}
    \forall \vU^\star \in \argmin_{\vU} \loss (\vU),
    \exists! \delta \in [0, 1] \text{ such that }  \vU^{\delta} = \vU^{\star}.
\end{align}
\end{theorem}

Theorem~\ref{thm:optimal:embedding} implies that we can identify the optimal embedding set as follows: we first represent the SupCL loss $\loss(\vU)$ in terms of $\delta$ (by substituting the inner products $\vu^T \vv$ and $\vu^T \vw$ in \eqref{eq:loss-sup} and \eqref{eq:loss-cl} with the right-hand sides of \eqref{eq:model:same}, \eqref{eq:model:same:class}, and \eqref{eq:model:different}), and then find the optimal $\delta^{\star}\in[0,1]$ that minimizes the loss. 

Note that the range of the optimal $\delta^\star$ in Theorem~\ref{thm:optimal:embedding} is $[0,1]$, which is a subset of $\Big[0,\sqrt\frac{mn-1}{m(n-1)}\Big]$ in Def.~\ref{def:model}. The following proposition and remark provide some implications obtained from the distinct ranges of $\delta$. 

\begin{proposition}
    \label{thm:ssem:similarity}
    Let $\vU^\delta$ be a set of embedding vectors forming \ours{}, as defined in Def.~\ref{def:model}, where $m$ and $n$ can be arbitrarily chosen. 
    For all $i\ne i' \in [m]$, $\vu^\delta, \vv^\delta \in \vU^{\delta}_i$ and $\vw^\delta \in \vU_{i'}^\delta$,
    \begin{equation}
    \nonumber
    (\vu^\delta)^\top \vv^\delta
    \geq (\vu^\delta)^\top \vw^\delta
    \end{equation}
    holds, if and only if $\delta\in [0, 1]$.
\end{proposition}

Proposition~\ref{thm:ssem:similarity} implies that for the optimal embedding set that minimizes the SupCL loss, the embeddings of the instances from the same class are always closer to each other, compared with the embeddings of the instances from different classes, which is desired. The proof of this proposition is in Appendix~\ref{appendix:properties:ssem}.

\begin{remark}
\citet{chen2022perfectly} propose a class-conditional version of the InfoNCE loss $\loss_{\op{cNCE}}(\vU)$, where negative pairs are restricted within each class, and combine it with $\loss_{\op{Sup}}(\vU)$ in \eqref{eq:loss-sup}, thus defining the loss as $(1-\alpha)\; \loss_{\op{Sup}}(\vU)+\alpha\; \loss_{\op{cNCE}}(\vU)$ for some $\alpha\in[0,1]$. In Appendix~\ref{appendix:proof:optimality:other}, we show that \ours{} with $\delta=\sqrt\frac{mn-1}{m(n-1)} \; (>1)$ is one of optimal embedding sets that minimize $\loss_{\op{cNCE}}(\vU)$, implying that the embedding vectors minimizing $\loss_{\op{cNCE}}(\vU)$ satisfy  $(\vu^\delta)^\top \vv^\delta < (\vu^\delta)^\top \vw^\delta$ for all $i\ne i' \in [m]$, $\vu^\delta, \vv^\delta \in \vU^{\delta}_i$ and $\vw^\delta \in \vU_{i'}^\delta$ from Proposition~\ref{thm:ssem:similarity}. In other words, training with the loss proposed by \citet{chen2022perfectly} may incur undesired result, where embeddings of instances in different classes are closer to each other, compared with embeddings of instances in the same class. Therefore, we recommend relying on the traditional self-supervised contrastive loss $\loss_{\op{Self}} (\vU)$ as in \eqref{eq:loss} instead of $\loss_{\op{cNCE}}(\vU)$, to avoid such issue.
\end{remark}

\section{PREVENTING 
EMBEDDINGS FROM CLASS COLLAPSE}
\label{sec:emb:var}

In this section, we investigate the optimal embedding sets from the perspective of variance~\citep{fisher1936use, rao1948utilization} and identify the conditions on the training settings in order to prevent the \textit{class collapse} where all embeddings of the instances in the same class collapses to a single vector.

\subsection{Variance of Embeddings}\label{sec:var-embedding}

For a given embedding set, we define two types of variances, within-class variance and between-class variance:

\begin{definition}[Variance of Embeddings]
    \label{def:emb:var}
    Let $\vU$ be a set of embedding vectors for $mn$ instances, and $\vU_i$ be the subset of $\vU$ corresponding to the embeddings for instances in $i$-th class, as in Sec.~\ref{sec:formulation}. 
    For all $i \in [m]$, the $i$-th \emph{within-class variance} of $\vU$ is defined as
    \begin{equation}
        \label{eq:emb:var:within}
        \Var[\vU_i] 
        := 
        \frac{1}{n}\sum_{\vu\in\vU_i} \norm{ \vu - \E [\vU_i]}_2^2,
    \end{equation}
    where $\E [\vU_i] :=\frac{1}{n} \sum_{\vu\in\vU_i} \vu$ is the expectation of the embedding vectors in $\vU_i$. We also define the \emph{between-class variance} of $\vU$ as
    \begin{equation}
        \label{eq:emb:var:between}
        \Var^{\op{Btwn}}[\vU] := \frac{1}{m}\sum_{i\in [m]}\norm{ \E [\vU_i]- \E [\vU]}_2^2.
    \end{equation}
\end{definition}

The variances in Def.~\ref{def:emb:var} are important metrics that capture the behavior of embedding vectors. When the $i$-th within-class variance $\Var[\vU_i]$ in \eqref{eq:emb:var:within} is zero, 
we cannot distinguish the instances in the same class, which is known as the \textit{class collapse}.  
When the between-class variance $\Var^{\op{Btwn}}[\vU]$ in \eqref{eq:emb:var:between} is zero, we cannot separate different classes. 
Therefore, we want both metrics to be large enough so that different embedding vectors are separable.
However, Proposition~\ref{thm:bounded:var} shows that the sum of these variances, which is also known as the total variance, is bounded, the proof of which is in Appendix~\ref{appendix:proof:var}.
\begin{proposition}[Bounded Variance]
    \label{thm:bounded:var}
    Let a set of embedding vectors $\vU$ in Sec.~\ref{sec:formulation} lie on the $d$-dimensional unit sphere, \ie $\|\vu\|_2^2 = 1$ for all $\vu\in\vU$. Then, the sum of all within-class variances and the between-class variance is bounded as
    \begin{equation}
        \label{eq:var:maximum}
         \frac{1}{m}\sum_{i\in [m]} \Var[\vU_i]+ \Var^{\op{Btwn}}[\vU] \leq 1.
    \end{equation}
    Here, the maximum is achieved when the centroid of the embedding vectors is at the origin, \ie $\E[\vU]=\vzero$.
\end{proposition}

We found that the embeddings forming \ours{} achieves the upper bound in~\eqref{eq:var:maximum}, as formally stated below. The proof of this proposition is in Appendix~\ref{appendix:proof:var}.

\begin{proposition}[Variance of \ours{}]
    \label{thm:ssem:var}
    Let $\vU^\delta$ be the embedding set forming $(m,n,\delta)$-\ours{} in Def.~\ref{def:model}. For any $\delta\in \Big[0,\sqrt\frac{mn-1}{m(n-1)}\Big]$,
    \begin{align*}
        &\Var \big[\vU_i^\delta\big] =\delta^2 \frac{m(n-1)}{mn-1}
        \qquad \forall i \in [m],
        \\&
        \Var^{\op{Btwn}} \big[\vU^\delta\big] =1 - \delta^2 \frac{m(n-1)}{mn-1}
    \end{align*}
    hold. Therefore, 
    \begin{equation*}
        \frac{1}{m} \sum_{i\in [m]} \Var \big[\vU_i^\delta\big] + \Var^{\op{Btwn}} \big[\vU^\delta\big] =1
        .
    \end{equation*}
\end{proposition}

\begin{remark}\label{remark:var}
The left-hand-side of \eqref{eq:var:maximum} is the sum of the within-class variance and the between-class variance of embeddings $U$. The within-class variance reflects the diversity of embeddings within each class, while the between-class variance corresponds to the separability between different classes.

Recall that SupCL has two objectives: (i) increasing the separation between instances from different classes and (ii) encouraging diversity among embeddings of instances within the same class. In this context, maximizing both variances is desirable, meaning the total variance should reach its maximum value. This maximum is attained when equality holds in \eqref{eq:var:maximum}.

Notably, by combining Theorem~\ref{thm:optimal:embedding} and Proposition~\ref{thm:ssem:var}, we conclude that the optimal embedding set $\vU^\star$ in \eqref{eq:optimal:embedding:general}, which minimizes the SupCL loss, always achieves the maximum variance in \eqref{eq:var:maximum}.
\end{remark}

\subsection{Preventing Class Collapse}

In this section, we discuss how to prevent the optimal embedding set (that minimizes the loss $\loss(\vU)$ in \eqref{eq:loss}) from the class collapse, by using the theoretical results provided in Sec.~\ref{sec:optimal:structure} and Sec.~\ref{sec:var-embedding}. Recall that as shown in Proposition~\ref{thm:ssem:var} and Remark~\ref{remark:var}, the optimal embedding set $\vU^{\star}$ minimizing the SupCL loss has the maximum variance (which is desired), and the balance between the within-class  variance $\Var\big[\vU_i^{\delta}\big]$ and the between-class variance $\Var^{\op{Btwn}} \big[\vU^\delta\big]$ is controlled by the parameter $\delta$; one can easily check that the class-collapse occurs when $\delta=0$. The below theorem provides the conditions on the loss-combining coefficient $\alpha$ and the temperature $\tau$, in order to avoid the class collapse; see Appendix~\ref{appendix:proof:class:collapse} for the proof.

\begin{theorem}[Preventing Class Collapse]
    \label{thm:prevent:collapse:alpha}
    Let $\vU^{\star}$ be the set of optimal embedding vectors that minimizes the loss $\loss(\vU)$ in \eqref{eq:loss}.
    Then, the class collapse does not happen, \ie $\Var[\vU_i^{\star}]>0$ for all $i\in[m]$, if and only if 
    the loss-combining coefficient $\alpha$ satisfies
    \begin{equation}
        \label{eq:alpha:prevent:collapse}
        \alpha
        \in \left(
        \frac{mn-1 + \exp\big( \frac{m}{m-1} / \tau\big)}{mn-n + n\cdot\exp\big( \frac{m}{m-1} / \tau\big)}
        , 1 \right]
    \end{equation}
    for a given temperature $\tau >0$.
    This necessary and sufficient condition for preventing class collapse can be re-written as 
    \begin{equation}
        \label{eq:temperature:non:collapse}
        \tau
        \in
        \left(0, 
        \frac{1}{
            \left(1-\frac{1}{m}\right) \cdot\log\left( \frac{mn-1 -\alpha (m-1)n}{ \alpha n - 1} \right)
        }
        \right)
    \end{equation}
    for a given $\alpha \in \big(\frac{1}{n},1\big]$.
\end{theorem}

The above theorem provides practical guidance for selecting appropriate hyperparameters ($\alpha$ and $\tau$) to guarantee that the embeddings trained with the loss do not suffer from class collapse. According to \eqref{eq:alpha:prevent:collapse}, the minimum value of $\alpha$ that prevents class collapse increases monotonically with respect to $\tau > 0$. For example, when $\tau=0.5$ and $m = 10$, with sufficiently large $n$, class collapse is prevented if $\alpha \in [0.549, 1]$. On the other hand, for $\tau=0.9$ and $m = 10$ with sufficiently large $n$, class collapse is avoided when using $\alpha \in [0.804, 1]$. This relationship is also illustrated in the red line in the top plot of Fig.~\ref{fig:withinvar}. Consequently, when a smaller $\tau$ is used, a wider range of $\alpha$ avoids class collapse. Furthermore, the minimum $\alpha$ required to prevent class collapse converges to $\frac{1}{n}$ as $\tau$ goes to zero, indicating that $\alpha$ must always be greater than $\frac{1}{n}$.

How about the condition on $\tau$ for a given $\alpha \in (\frac{1}{n},1]$? According to \eqref{eq:temperature:non:collapse}, the maximum temperature parameter $\tau$ (to avoid class collapse) converges to $\big((1-\frac{1}{m})\cdot\log(1+ \frac{1-\alpha}{\alpha}m )\big)^{-1}$, in the asymptotic regime of large $n$. In the standard setting of $\alpha=0.5$, this condition reduces to $\tau \lesssim \frac{1}{\log m}$. This suggests that the convention $\tau=0.1$ (\eg{} \citet{khosla2020supervised}) is a reasonable choice, unless the number of classes is extremely large, \ie $m \ge \exp(10)$.

All in all, the results in Theorem~\ref{thm:prevent:collapse:alpha} imply that to guarantee the learned embeddings do not suffer from class collapse, it is necessary to satisfy $\alpha \ge \frac{1}{n}$ and $\tau \le \frac{1}{\log m}$.

\section{EXPERIMENTS}
\label{sec:experiment}

In this section, we provide experimental results showing that our theoretical analysis on the within-class variance of learned embeddings (given in Sec.~\ref{sec:optimal:structure} and Sec.~\ref{sec:emb:var}) 
provides practical guidelines on configurations of supervised contrastive learning to avoid class collapse. We run experiments on synthetic datasets and real image datasets, where the source code is given in \url{https://github.com/leechungpa/ssem-supcl}.

\subsection{Experiments on Synthetic Data}
\label{sec:experiment:synthetic}

\begin{figure}[t]
    \centering
    \includegraphics[width=0.48\textwidth]{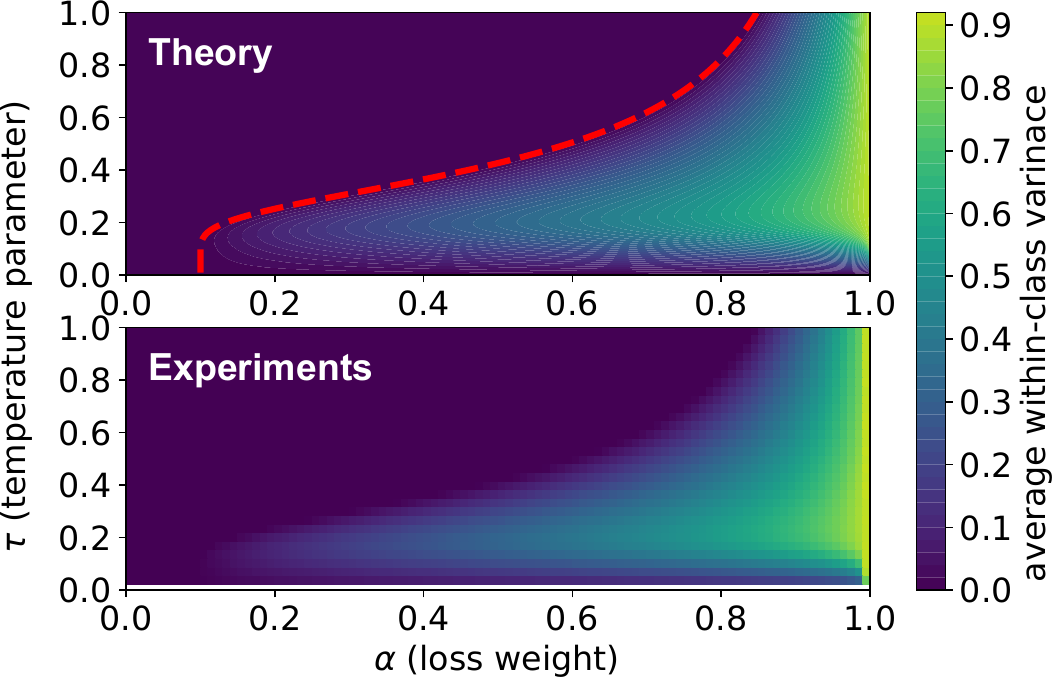}
    \caption{
    The within-class variance (averaged over different classes) of the learned embedding set ($\frac{1}{m} \sum_{i\in [m]} \Var \big[\vU_i\big]$ in~\eqref{eq:emb:var:within}), for various loss-combining coefficient $\alpha$ and temperature $\tau$. (Top): Derived from theoretical results in Sec.~\ref{sec:emb:var}, (Bottom): Computed from the experiments on synthetic datasets in Sec.~\ref{sec:experiment:synthetic}.
    One can confirm that both results (shown at the top and the bottom figures) are well aligned. 
    Here, the red dashed line at the top figure indicates the boundary of regions having zero within-class variance, \ie when class collapse happens.
    }
    \label{fig:withinvar}
\end{figure}

\begin{figure*}[t]
    \centering
    \begin{tabular}{ccc}
    \includegraphics[width=0.3\textwidth]{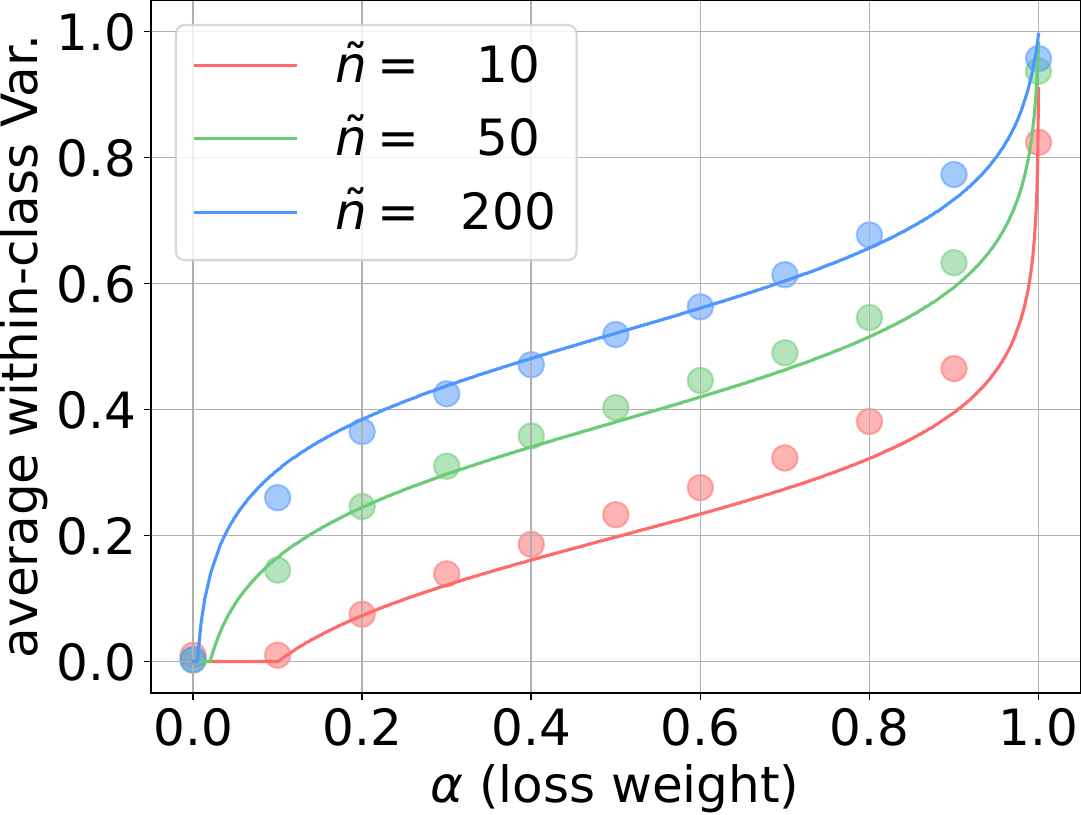} &
    \includegraphics[width=0.3\textwidth]{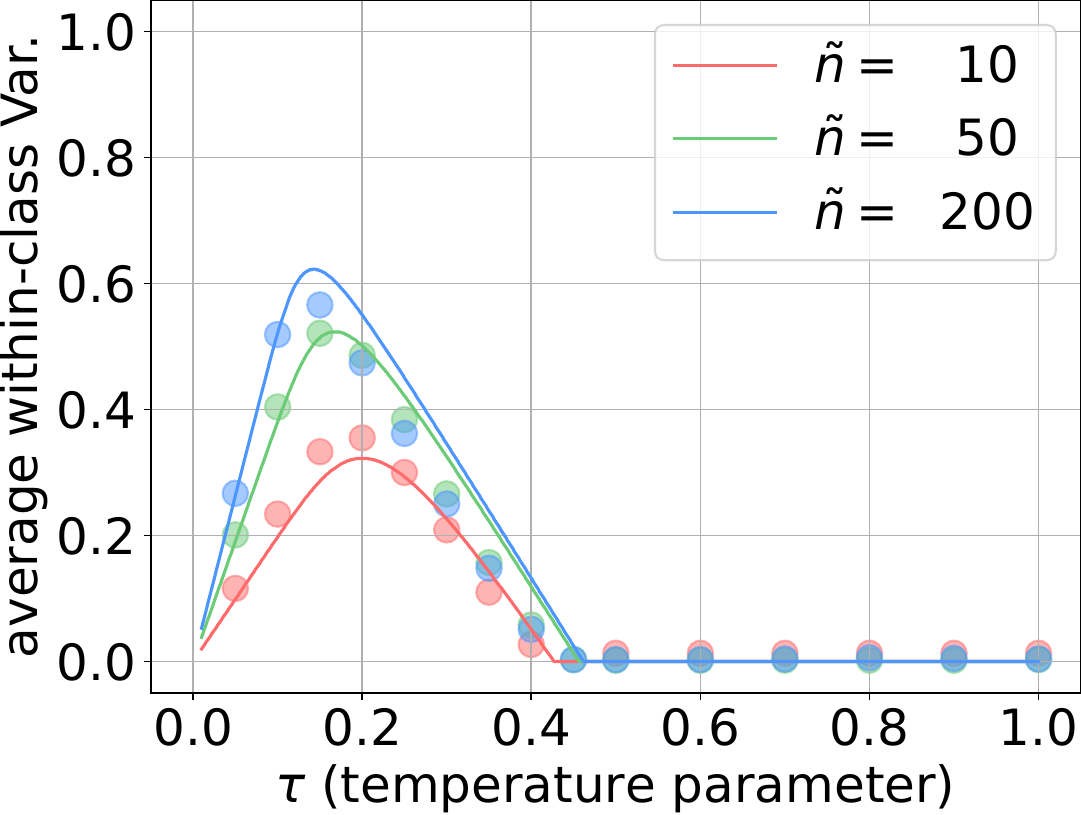} & 
    \includegraphics[width=0.31\textwidth]{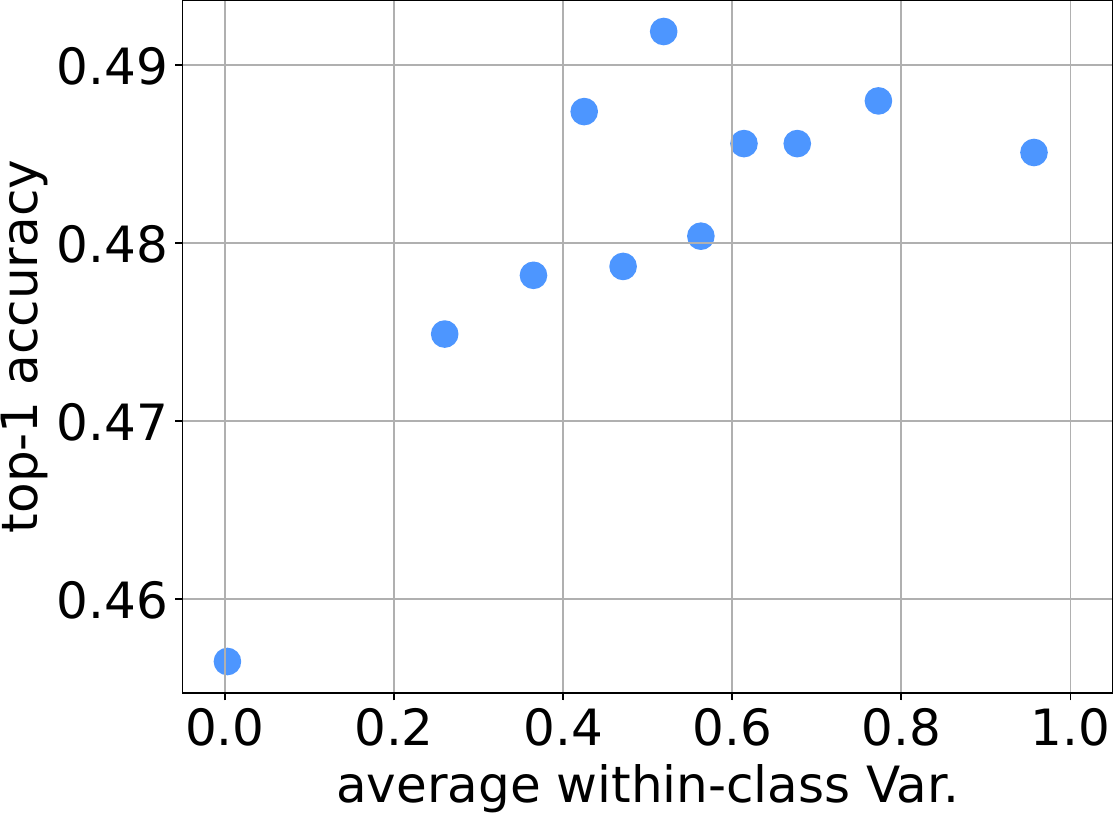} \\
     (a)  & (b) & (c) \\
    \end{tabular}
    \caption{
    Average within-class variance of the learned embeddings obtained in theory (lines) and by experiments (dots), measured on CIFAR-10 dataset when ResNet-18 encoder is used. 
    (a), (b): Dependency of the within-class variance on $\alpha$ and $\tau$, for various per-class batch sizes $\tilde{n}=10, 50, 200$. The values obtained in experiments match with those computed from our theoretical results. (c): Relationship between the within-class variance of the learned embeddings and the transfer learning performance (when transferred to CIFAR-100) of the CIFAR-10 trained ResNet-18 encoder. We set $\tau=0.1$ and $\tilde{n}=200$, and run experiments on various $\alpha=0.0, 0.1, \cdots, 1.0$. Note that embeddings having a moderate amount of within-class variances achieves the highest performance.
    }
    \label{fig:cifar}
    \vspace{8pt}
\end{figure*}

\paragraph{Setup.}
Following previous works that reported experimental results on CL using synthetic datasets~\citep{sreenivasan2023mini,lee2024analysis}, we consider the scenario of \textit{directly} optimizing the embedding set $\vU$, instead of training an encoder that maps from each data to its embedding vector. Initially, all embedding vectors in $\vU$ are randomly sampled from the $100$-dimensional standard Gaussian distribution and then normalized to ensure unit norm. We then optimize these vectors to minimize the SupCL loss $\loss(\vU)$ in \eqref{eq:loss} over $1,000$ epochs using the Adam optimizer, where the learning rate is set to $0.5$. Note that at each training step, the updated embeddings are projected back onto the unit sphere to ensure that they maintain unit norm.
Here, we have $\lvert \vU \rvert = mnp=200$ embeddings, which are categorized as $m=10$ classes, $n=10$ instances per class, and $p=2$ augmentations per instance.

\paragraph{Within-Class Variance of Learned Embeddings.}
Fig.~\ref{fig:withinvar} shows the average within-class variance 
($\frac{1}{m} \sum_{i\in [m]} \Var \big[\vU_i\big]$ in~\eqref{eq:emb:var:within}) of the learned embedding set that minimizes the SupCL loss in~\eqref{eq:loss}, for various $\alpha$ and $\tau$. The top figure presents the average within-class variance derived from theoretical results (based on Theorem~\ref{thm:optimal:embedding} and Proposition~\ref{thm:ssem:var}), whereas the bottom figure shows the average within-class variance computed from embeddings trained on our synthetic datasets. At each $(\alpha, \tau)$ pair, the average within-class variance obtained in theory well matches with those computed in experiments. These results show that our analysis on the variance of learned embeddings in Sec.~\ref{sec:emb:var} is valid for synthetic datasets.

\subsection{Experiments on Real Data}
\label{sec:experiment:real}

\paragraph{Setup.}

We run experiments on two image datasets: CIFAR-10~\citep{krizhevsky2009cifar} and ImageNet-100~\citep{deng2009imagenet}.
We use ResNet architecture~\citep{he2016deep} for the encoder, and 2-layer MLP for the projector, where the output dimension of the projector is set to $d=128$.

After training embeddings using the SupCL loss $\mathcal{L}(\vU)$ in~\eqref{eq:loss}, we measure the average within-class variance ($\frac{1}{m} \sum_{i\in [m]} \Var \big[\vU_i\big]$ in~\eqref{eq:emb:var:within}) of the embeddings (normalized output of the projector) of the train data, without any augmentation. 
We then remove the projector head and evaluate the performance of pretrained encoder on various downstream classification tasks by using linear probing. Following recent works~\citep{kornblith2019better, lee2021improving, oh2024effectiveness}, we evaluate the top-1 accuracy or the mean per class accuracy, depending on the downstream tasks. Details of the datasets, as well as the training and evaluation processes, are provided in Appendix~\ref{appendix:real:exp}.

\begin{table*}[t]
\caption{
Transfer learning performance evaluated on various downstream tasks. We first train a ResNet-50 encoder on ImageNet-100 by minimizing the SupCL loss $\mathcal{L}(\mathbf{U})$ in \eqref{eq:loss} with $\tau=0.1$ for various $\alpha$ values.
Then, we evaluate the pretrained encoder on multiple downstream classification tasks using linear probing, measuring the performance via top-1 accuracy or mean per-class accuracy. 
Here, for evaluating the second column of the table, we compute the within-class variance of each class using the embeddings of the original training data (without any augmentations), and take average of the within-class variances. 
The best result for each downstream dataset is highlighted in bold. Notably, the model with a moderate amount of average within-class variance shows the highest average performance.
}
\label{table:transfer:full}
\begin{center}
\resizebox{\textwidth}{!}{%
\begin{tabular}{c|c|c|ccccccccccc}
\toprule
$\alpha$ & Avg. within-class var. & Avg. accuracy & \multicolumn{1}{l}{CIFAR10} & \multicolumn{1}{l}{CIFAR100} & \multicolumn{1}{l}{Caltech101} & \multicolumn{1}{l}{CUB200} & \multicolumn{1}{l}{Dog} & \multicolumn{1}{l}{DTD} & \multicolumn{1}{l}{Flowers102} & \multicolumn{1}{l}{Food101} & \multicolumn{1}{l}{MIT67} & \multicolumn{1}{l}{Pets} & \multicolumn{1}{l}{SUN397} \\
\midrule
0.0 & 0.014 & 68.00 & 88.34 & 68.25 & \textbf{88.87} & 36.05 & 62.04 & 65.64 & 88.46 & 57.98 & 63.13 & 79.05 & 50.19 \\
0.2 & 0.021 & 68.16 & 88.70 & 68.96 & 87.35 & 35.45 & 62.10 & 65.11 & 88.46 & 59.26 & 64.40 & 80.09 & 49.89 \\
0.4 & 0.078 & 68.42 & 88.91 & 69.08 & 87.60 & 37.35 & 62.39 & 66.33 & 88.19 & 59.78 & 63.28 & 79.56 & 50.20 \\
0.5 & 0.133 & \textbf{69.06} & \textbf{89.43} & 69.45 & 88.35 & \textbf{38.48} & \textbf{62.78} & 66.33 & \textbf{89.49} & 60.36 & 63.43 & \textbf{80.68} & 50.88 \\
0.6 & 0.192 & 68.72 & 89.34 & \textbf{70.07} & 87.50 & 35.99 & 61.56 & 66.33 & 87.82 & 61.03 & \textbf{65.45} & 79.75 & \textbf{51.13} \\
0.8 & 0.325 & 68.36 & 88.13 & 68.42 & 86.06 & 36.46 & 60.73 & \textbf{66.49} & 89.29 & \textbf{62.72} & 64.55 & 78.48 & 50.60 \\
1.0 & 0.830 & 63.06 & 84.15 & 63.17 & 78.64 & 30.40 & 46.19 & 65.00 & 85.71 & 62.02 & 62.91 & 67.60 & 47.89 \\
\bottomrule
\end{tabular}%
}
\end{center}
\end{table*}

\paragraph{Within-Class Variance of Learned Embeddings.}

Fig.~\ref{fig:cifar} shows the average within-class variance measured for CIFAR-10 dataset. To be specific, Fig.~\ref{fig:cifar}a shows the dependency on the loss-combining coefficient $\alpha$, when temperature is set to $\tau=0.1$, while Fig.~\ref{fig:cifar}b shows the dependency on $\tau$, when $\alpha=0.5$.

We test on three different batch sizes: the per-class batch size (\ie the number of instances from the same class and contained in the same batch) is set to $\tilde{n}=10, 50, 200$. Here, we use balanced batches, where the number of instances per class is equal in every batch. Since CIFAR-10 datasets are categorized into $m=10$ classes, we have batch sizes of $m\tilde{n} = 100, 500, 2000$, respectively.

In Fig.\ref{fig:cifar}a and Fig.\ref{fig:cifar}b, we compare two types of within-class variances: one computed based on the theoretical results in Sec.\ref{sec:emb:var} (shown as solid lines) and the other obtained from experiments (shown as dots). Note that in order to reflect the batch effect, we use $\tilde{n}$ instead of $n$ in the expressions given in Sec.\ref{sec:emb:var} when plotting the solid line representing the theoretical results.

One can observe that for each $\tilde{n}$, the solid line (derived from theory) aligns with the dots (obtained from experiments), confirming that our analysis on the variance of learned embeddings in Sec.~\ref{sec:emb:var} is valid for real datasets.

\paragraph{Relationship between Within-Class Variance and Transfer Learning Performance.}

Now we provide experimental results showing that the within-class variance of the learned embeddings is a good indicator of the transfer learning performance. Fig.~\ref{fig:cifar}c shows the relationship between the within-class variance and the transfer learning performance, when the embeddings are used to classify CIFAR-100 datasets. Here, we set $\tau=0.1$ and $\tilde{n}=200$, and evaluate on various $\alpha$ values ranging from 0 to 1. 

The result in Fig.~\ref{fig:cifar}c indicates that embeddings suffering from class collapse, which have zero within-class variance, exhibit the lowest top-1 accuracy on CIFAR-100 classification. In contrast, embeddings with a moderate amount of within-class variance (roughly between $0.4$ and $0.8$) achieve the highest top-1 accuracy on CIFAR-100 classification.

This trend, where the best performance occurs when embeddings maintain a moderate amount of within-class variance, is also observed in more complex architectures and datasets. 
In Table~\ref{table:transfer:full}, we report the relationship between the transfer learning performance and the within-class variance for the embeddings learned on ImageNet-100 dataset. 
Specifically, we train a ResNet-50 encoder on ImageNet-100 using the SupCL loss $\mathcal{L}(\mathbf{U})$ in \eqref{eq:loss} with $\tau=0.1$ for various $\alpha$ values, resulting in each trained encoder exhibiting a different within-class variance. We then evaluate its transfer learning performance across different downstream datasets, as well as the within-class variance of the trained model. 
In addition, we run experiments on  other transfer learning tasks including object detection and few-shot learning tasks, results of which are provided in Appendix~\ref{sec:appendix:transfer}. 
Across all these tasks, the findings consistently demonstrate that embeddings with a moderate amount of within-class variance yield the best performance.

\section{CONCLUSION}
This paper explores the behavior of embeddings trained with supervised contrastive learning, from the perspective of variance. First, we prove that Simplex-to-Simplex Embedding Model (\ours{}), a class of embedding sets defined by us, contains the optimal embedding set that minimizes the supervised contrastive loss. Then, we provide theoretical analysis on the behaviors of optimal embeddings from the perspective of within-class and between-class variances and offer guidelines on training configurations to prevent the learned embeddings from class collapse, \ie when the within-class variance is zero. This theoretical result aligns well with our experimental findings on synthetic and real datasets. Specifically, the within-class variance of the learned embeddings (computed in our experiments) matches the mathematical expressions derived in theory. Moreover, our theory-driven guideline for preventing class collapse is empirically validated, showing consistency with conventional practices.

\subsubsection*{Acknowledgements}

This work was supported by the National Research Foundation of Korea (NRF) grant funded by the Korea government (MSIT) (RS-2024-00341749, RS-2024-00345351, RS-2024-00408003), the MSIT (Ministry of Science and ICT), Korea, under the ICAN (ICT Challenge and Advanced Network of HRD) support program (RS-2023-00259934) supervised by the IITP (Institute for Information \& Communications Technology Planning \& Evaluation), the Yonsei University Research Fund (2024-22-0124, 2024-22-0148).

\clearpage
\bibliography{__ref}

\clearpage
\section*{Checklist}

 \begin{enumerate}

 \item For all models and algorithms presented, check if you include:
 \begin{enumerate}
   \item A clear description of the mathematical setting, assumptions, algorithm, and/or model. [Yes] See Sec~\ref{sec:formulation}.
   \item An analysis of the properties and complexity (time, space, sample size) of any algorithm. [Not Applicable]
 \end{enumerate}

 \item For any theoretical claim, check if you include:
 \begin{enumerate}
   \item Statements of the full set of assumptions of all theoretical results. [Yes]
   Following prior works, we use a normalized encoder $f$, \ie $\lVert f(\vx) \rVert_2 = 1$ for all input $\vx$, as described in Sec~\ref{sec:formulation}. To streamline the notation, we present our theoretical results in the main paper for the case of a single augmentation ($p=1$). These results are extended to the general case ($p>1$) in Appendix~\ref{appendix:proof}.

   \item Complete proofs of all theoretical results. [Yes] See Appendix~\ref{appendix:proof}.
   \item Clear explanations of any assumptions. [Yes] See Sec~\ref{sec:formulation}.
 \end{enumerate}

 \item For all figures and tables that present empirical results, check if you include:
 \begin{enumerate}
   \item The code, data, and instructions needed to reproduce the main experimental results (either in the supplemental material or as a URL). [Yes] See Appendix~\ref{appendix:real:exp}.
   \item All the training details (e.g., data splits, hyperparameters, how they were chosen). [Yes] See Appendix~\ref{appendix:real:exp}.
   \item A clear definition of the specific measure or statistics and error bars (e.g., with respect to the random seed after running experiments multiple times). [Yes] See Appendix~\ref{appendix:real:exp}.
   \item A description of the computing infrastructure used. (e.g., type of GPUs, internal cluster, or cloud provider). [Yes] See Appendix~\ref{appendix:real:exp}.
 \end{enumerate}

 \item If you are using existing assets (e.g., code, data, models) or curating/releasing new assets, check if you include:
 \begin{enumerate}
   \item Citations of the creator If your work uses existing assets. [Yes] See Appendix~\ref{appendix:real:exp}.
   \item The license information of the assets, if applicable. [Not Applicable]
   \item New assets either in the supplemental material or as a URL, if applicable. [Not Applicable]
   \item Information about consent from data providers/curators. [Not Applicable]
   \item Discussion of sensible content if applicable, e.g., personally identifiable information or offensive content. [Not Applicable]
 \end{enumerate}

 \item If you used crowdsourcing or conducted research with human subjects, check if you include:
 \begin{enumerate}
   \item The full text of instructions given to participants and screenshots. [Not Applicable]
   \item Descriptions of potential participant risks, with links to Institutional Review Board (IRB) approvals if applicable. [Not Applicable]
   \item The estimated hourly wage paid to participants and the total amount spent on participant compensation. [Not Applicable]
 \end{enumerate}

 \end{enumerate}

\newpage
\appendix
\numberwithin{table}{section}
\numberwithin{figure}{section}
\numberwithin{equation}{section}
\onecolumn

\clearpage
\section{PROOFS}
\label{appendix:proof}

\subsection{Proposition About Embedding Vectors on the Unit Sphere}
\label{appendix:proof:var}

The definitions of within-class and between-class variances are reformulated as follows:
\begin{definition}[Restatement of Def.~\ref{def:emb:var}]
    Let $\vU$ be a set of embedding vectors, and $\vU_i$ be the subset of $\vU$ corresponding to the embeddings for instances in $i$-th class, as in Sec.~\ref{sec:formulation}. 
    For all $i \in [m]$, the $i$-th \emph{within-class variance} of $\vU$ is defined as
    \begin{equation}
        \nonumber
        \Var[\vU_i] 
        := 
        \frac{1}{|\vU_i|}\sum_{\vu\in\vU_i} \norm{ \vu - \E [\vU_i]}_2^2,
    \end{equation}
    where $\E [\vU_i] :=\frac{1}{|\vU_i|} \sum_{\vu\in\vU_i} \vu$ is the expectation of the embedding vectors in $\vU_i$. We also define the \emph{between-class variance} of $\vU$ as
    \begin{equation}
        \nonumber
        \Var^{\op{Btwn}}[\vU] := \sum_{i\in [m]} \frac{|\vU_i|}{|\vU|} \norm{ \E [\vU_i]- \E [\vU]}_2^2.
    \end{equation}
\end{definition}

\begin{proposition}
    \label{thm:app:emb:var:on:sphere}
    Consider a set of $n$ vectors $\vU$ on the sphere of radius $r>0$, \ie $\|\vu\|_2 = r$ for all $\vu \in \vU$. Then, the variance of set $\vU$ can be rewritten as
    \begin{align*}
        \Var[\vU]
        &= 
        r^2 -   \| \E [\vU]\|_2^2
        =
        \frac{|\vU|-1}{|\vU|} \cdot r^2 -  \frac{1}{|\vU|^2}  \sum_{\vu\in\vU} \sum_{\vv\in\vU\setminus\{\vu\}} \vu^\top \vv.
    \end{align*}
\end{proposition}
\begin{proof}
According to Def.~\ref{appendix:proof:var}, the variance of vectors on the unit sphere is determined as follows:
\begin{align*}
    \Var[\vU] 
    &= 
    \frac{1}{|\vU|}\sum_{\vu\in\vU} \| \vu - \E [\vU]\|_2^2
    \\&=
    \frac{1}{|\vU|}\sum_{\vu\in\vU} \| \vu\|_2^2 
    + \| \E [\vU]\|_2^2
    -  \frac{2}{|\vU|}\sum_{\vu\in\vU} \vu^\top \E [\vU]
    \\&=
    \frac{1}{|\vU|}\sum_{\vu\in\vU} \| \vu\|_2^2 -  \| \E [\vU]\|_2^2
    \\&=
    r^2 -   \| \E [\vU]\|_2^2
    \\&=
    r^2 - \frac{1}{|\vU|^2} \bigg\| \sum_{\vu\in\vU}\vu\bigg\|_2^2
    \\&=
    r^2 - \frac{1}{|\vU|^2} \sum_{\vu\in\vU} \sum_{\vv\in\vU} \vu^\top \vv
    \\&=
    \frac{|\vU|-1}{|\vU|} \cdot r^2 - \frac{1}{|\vU|^2}  \sum_{\vu\in\vU} \sum_{\vv\in\vU\setminus\{\vu\}} \vu^\top \vv.
\end{align*}
\end{proof}

\begin{proposition}[Restatement of Proposition~\ref{thm:bounded:var}]
    \label{thm:app:bounded:var}
    For $i\in[m]$, let $\vU_i$ be the set of vectors on the sphere with radius $r>0$, \ie  $\|\vu\|_2 = r$ for all $i\in[m]$ and $\vu \in \vU_i$. Define the entire set as $\vU := \cup_{i\in[m]} \vU_i$. Then, the sum of variances is bounded as
    \begin{equation*}
        \sum_{i\in [m]}\frac{|\vU_i|}{|\vU|} \Var[\vU_i] + \Var^{\op{Btwn}}[\vU] \leq r^2.
    \end{equation*}
    The maximum is achieved when the centroid of the vectors is at the origin; that is, 
    \begin{equation*}
        \E[\vU]=\vzero.
    \end{equation*}
\end{proposition}

\begin{proof}
For simplicity of notation, let $n_i=|\vU_i|$ and $N=\sum_{i\in[m]}n_i$. Note that the sum of following inner products is zero.
\begin{align*}
    \sumi\sum_{\vu\in\vU_i}
    \big( \vu - \E [\vU_i]\big)^\top \big( \E [\vU_i]- \E [\vU] \big)
    &=
    \sumi\left( \sum_{\vu\in\vU_i}\vu - n_i\cdot\E [\vU_i]\right)^\top \big( \E [\vU_i]- \E [\vU] \big)
    \\ &=\sumi\vzero ^\top \big( \E [\vU_i]- \E [\vU] \big)
    \\ &=0.
\end{align*}
Then, we can decompose the summation of norms as follows.
\begin{align*}
    \sum_{\vu\in\vU} \norm{\vu - \E [\vU]}_2^2
    &=
    \sum_{i\in[m]}\sum_{\vu\in\vU_i} \norm{\vu - \E [\vU]}_2^2
    \\&=
    \sum_{i\in[m]}\sum_{\vu\in\vU_i} \norm{  \vu - \E [\vU_i]+ \E [\vU_i]- \E [\vU]}_2^2
    \\&=\sum_{i\in[m]}\sum_{\vu\in\vU_i}\norm{ \vu - \E [\vU_i] }_2^2
    + \sum_{i\in[m]}\sum_{\vu\in\vU_i} \norm{ \E [\vU_i]- \E [\vU] }_2^2
    + \sumi\sum_{\vu\in\vU_i}
    2\big( \vu - \E [\vU_i]\big)^\top \big( \E [\vU_i]- \E [\vU] \big)    
    \\&=
    \sum_{i\in[m]}\sum_{\vu\in\vU_i}\norm{ \vu - \E [\vU_i] }_2^2
    + \sum_{i\in[m]}\sum_{\vu\in\vU_i} \norm{  \E [\vU_i]- \E [\vU] }_2^2  
    \\&=
    \sum_{i\in[m]}\sum_{\vu\in\vU_i}\norm{ \vu - \E [\vU_i] }_2^2
    + \sum_{i\in[m]}n_i \norm{  \E [\vU_i]- \E [\vU] }_2^2.
\end{align*}

As a result, the variance of the entire set can be rewritten as follows.
\begin{align*}
    \Var (\vW) 
    &=
    \frac{1}{N}
    \sum_{\vu\in\vU} \norm{\vu - \E [\vU]}_2^2
    \\ &=
    \frac{1}{N}
    \sum_{i\in[m]}\sum_{\vu\in\vU_i}\norm{ \vu - \E [\vU_i] }_2^2
    + \frac{1}{N}\sum_{i\in[m]} n_i \norm{ \E [\vU_i]- \E [\vU] }_2^2
    \\ &=
    \sum_{i\in [m]}\frac{n_i}{N} \Var[\vU_i] + \sum_{i\in[m]}\frac{n_i}{N}  \norm{ \E [\vU_i]- \E [\vU] }_2^2
    \\ &=
     \sum_{i\in [m]}\frac{n_i}{N} \Var[\vU_i] + \Var^{\op{Btwn}}[\vU]
    . 
\end{align*}

To establish the upper bound, from Proposition~\ref{thm:app:emb:var:on:sphere},
\begin{align}
    \Var[\vU] 
    &= \nonumber
    r^2 - \frac{1}{n^2} \norm{\sum_{\vu\in\vU}\vu}_2^2
    \\ &\leq r^2, \label{eq:thm:var:equal}
\end{align}
which implies
\begin{equation*}
    \sum_{i\in [m]}\frac{n_i}{N} \Var[\vU_i] + \Var^{\op{Btwn}}[\vU] \leq r^2.
\end{equation*}
The equality condition in \eqref{eq:thm:var:equal} is 
    \begin{align*}
        \E[\vU] &= \frac{1}{N}\sum_{\vu\in\vU} \vu =\vzero.
    \end{align*}
\end{proof}

\newpage
\subsection{Properties of \ours{}}
\label{appendix:properties:ssem}

We redefine the the Simplex-to-Simplex Embedding
Model (\ours{}) for a general $p\in \sN$ as follows:

\begin{definition}[Restatement of Def.~\ref{def:model}]
\label{def:model:general}
Let positive integers $m$,$n$,$p$, and a real non-negative number $\delta \in \Big[0,\sqrt\frac{mn-1}{m(n-1)}\Big]$ be given.  
We define Simplex-to-Simplex Embedding Model, denoted by ($m,n,p,\delta$)-\ours{}, as the set of $mnp$ vectors
\begin{align*}
    \vU^\delta = 
    \big\{\vu_{i,j,k}^{\delta}\big\}_{i\in[m], j \in [n], k\in[p]}
\end{align*}
satisfying the following:

For all $i \ne i'\in [m]$ and $j\ne j'\in[n]$,
\begin{align}
    (\vu^\delta)^\top \vv^\delta
    & \label{eq:model:same:general}
    = 1 
    &\forall \vu^\delta, \vv^\delta \in \vU_{i,j}^\delta,
    \\
    (\vu^\delta)^\top \vv^\delta
    & \label{eq:model:same:class:general}
    = 1 - \delta^2 \frac{mn}{mn-1}
    & \forall \vu^\delta \in \vU_{i,j}^\delta, \vv^\delta \in \vU_{i,j'}^\delta,
    \\
    (\vu^\delta)^\top \vv^\delta
    &= \label{eq:model:different:general}
    - \frac{1}{m-1} + \delta^2 \frac{m(n-1)}{(m-1)(mn-1)}
    & \forall \vu^\delta \in \vU^{\delta}_i, \vv^\delta \in \vU_{i'}^\delta,
\end{align}
where $\vU_{i,j}^\delta := \{\vu_{i,j,k}^\delta \}_{k \in[p]}$ and  $\vU^{\delta}_i: = \cup_{j\in[n]} \vU_{i,j}^\delta$.
\end{definition}

\begin{proposition}
    \label{thm:app:ssem:similarity}
    Let $\vU^\delta$ be a set of embedding vectors forming \ours{}, as defined in Def.~\ref{def:model:general}, where $m, n$ and $p$ can be arbitrarily chosen. 
    For all $i\ne i' \in [m]$, $\vu^\delta, \vv^\delta \in \vU^{\delta}_i$ and $\vw^\delta \in \vU_{i'}^\delta$,
    \begin{equation}
    \label{eq:prop:similarity:ineq}
    (\vu^\delta)^\top \vv^\delta
    \geq (\vu^\delta)^\top \vw^\delta
    \end{equation}
    holds, if and only if $\delta\in [0, 1]$.
\end{proposition}
\begin{proof}
     For all $i\ne i' \in [m]$, $\vu, \vv \in \vU^{\delta}_i$ and $\vw \in \vU_{i'}^\delta$, the following holds from the definition of \ours{}:
    \begin{align*}
        (\vu^\delta)^\top \vv^\delta
        &=
        \begin{cases}
            1
            & \text{if there exist }
            j\in[m]
            \text{ such that }
            \vu^\delta, \vv^\delta \in \vU_{i,j}^\delta
            \\
            1 - \delta^2 \frac{mn}{mn-1}  
            & \text{otherwise}
        \end{cases}
        \\ &\geq
        \min \left( 1, 1 - \delta^2 \frac{mn}{mn-1}   \right)
        \\ &=
        1 - \delta^2 \frac{mn}{mn-1},
        \\(\vu^\delta)^\top \vw^\delta 
        &=
        - \frac{1}{m-1} + \delta^2 \frac{m(n-1)}{(m-1)(mn-1)}.
    \end{align*}

    To determine the necessary and sufficient condition for \eqref{eq:prop:similarity:ineq},
    \begin{align*}
    (\vu^\delta)^\top \vv^\delta
    - (\vu^\delta)^\top \vw^\delta
        & \geq
        1 - \delta^2 \frac{mn}{mn-1}
        + \frac{1}{m-1} - \delta^2 \frac{m(n-1)}{(m-1)(mn-1)}
        \\ &=
        \frac{m}{m-1} - \delta^2 \frac{m}{m-1}
        \\ &=
        \frac{m}{m-1} (1- \delta^2),
    \end{align*}
    which implies that \eqref{eq:prop:similarity:ineq} holds if and only if $\delta \in [0, 1]$, as $\delta$ is defined to be within the range $\Big[0,\sqrt{\frac{mn-1}{m(n-1)}}\Big]$.
\end{proof}

\newpage
\begin{proposition}
    \label{thm:app:ssem:variance}
    Let $\vU^\delta$ be the embedding set forming $(m,n,p,\delta)$-\ours{} in Def.~\ref{def:model:general}. For any $\delta\in \Big[0,\sqrt\frac{mn-1}{m(n-1)}\Big]$,
    \begin{align*}
        &\Var \big[\vU_i^\delta\big] =\delta^2 \frac{m(n-1)}{mn-1}
        \qquad \forall i \in [m],
        \qquad\text{and}\qquad
        \Var^{\op{Btwn}} \big[\vU^\delta\big] =1 - \delta^2 \frac{m(n-1)}{mn-1}
    \end{align*}
    Therefore, 
    \begin{equation*}
        \frac{1}{m} \sum_{i\in [m]} \Var \big[\vU_i^\delta\big] + \Var^{\op{Btwn}} \big[\vU^\delta\big] =1
        .
    \end{equation*}
\end{proposition}
\begin{proof}
    Note that the followings hold from the definition of \ours{} in Def.~\ref{def:model:general}:
    
    For all $i\in [m]$,
    \begin{align*}
        \norm{\sum_{\vu^\delta \in \vU^{\delta}_i} \vu^\delta}_2^2    
        &= 
        \bigg(\sum_{\vu^\delta \in \vU^{\delta}_i} \vu^\delta\bigg)^\top \bigg(\sum_{\vu^\delta \in \vU^{\delta}_i} \vu^\delta\bigg)
        \\&= np^2 + n(n-1)p^2 \bigg(  1 - \delta^2 \frac{mn}{mn-1}  \bigg)
        \\&= n^2p^2 - \delta^2 \frac{mn^2(n-1)p^2}{mn-1},
        \\
        \bigg(\sum_{\vu^\delta \in \vU^{\delta}_i} \vu^\delta\bigg)^\top \bigg(\sum_{\vu^\delta \in \vU^{\delta}} \vu^\delta\bigg)
        &= np^2 + n(n-1)p^2 \bigg(  1 - \delta^2 \frac{mn}{mn-1}  \bigg)
        \\&\quad
        +(m-1)n^2p^2 \bigg( - \frac{1}{m-1} + \delta^2 \frac{m(n-1)}{(m-1)(mn-1)} \bigg)
        \\&= - \delta^2 \frac{mn^2(n-1)p^2}{mn-1}
        + \delta^2 \frac{mn^2(n-1)p^2}{mn-1} 
        \\&= 0,
        \\
        \bigg(\sum_{\vu^\delta \in \vU^{\delta}} \vu^\delta\bigg)^\top \bigg(\sum_{\vu^\delta \in \vU^{\delta}} \vu^\delta\bigg)
        &= \sum_{i\in[m]}\bigg(\sum_{\vu^\delta \in \vU^{\delta}_i} \vu^\delta\bigg)^\top \bigg(\sum_{\vu^\delta \in \vU^{\delta}} \vu^\delta\bigg)
        \\&= 0,
    \end{align*}
    which implies
    \begin{align*}
        \E \big[\vU^\delta\big]
        &=
        \frac{1}{mnp}\sum_{\vu^\delta \in \vU^{\delta}} \vu^\delta
        =
        \vzero
    \end{align*}
    for any $\delta\in \Big[0,\sqrt\frac{mn-1}{m(n-1)}\Big]$.
    Then, the between-class variance of \ours{} is determined as below.
    \begin{align*}
        \Var^{\op{Btwn}} \big[\vU^{\delta}\big] 
        &=
        \sum_{i\in [m]} \frac{np}{mnp}\norm{ \E \big[\vU_i^\delta\big]- \E \big[\vU^\delta\big]}_2^2
        \\&=
        \frac{1}{m} 
         \sum_{i\in[m]}\norm{ \frac{1}{np} \sum_{\vu^\delta\in\vU^{\delta}_i} \vu^\delta }_2^2
        \\&=
         \sum_{i\in[m]}\frac{1}{m} \bigg(
            \frac{1}{n^2p^2} 
            \bigg(  n^2p^2 - \delta^2 \frac{mn^2(n-1)p^2}{mn-1} \bigg)
            +0
         \bigg)
        \\&=
         \sum_{i\in[m]}\frac{1}{m} \bigg(
            1 - \delta^2 \frac{m(n-1)}{mn-1}
         \bigg)
        \\&=
         1 - \delta^2 \frac{m(n-1)}{mn-1}.
    \end{align*}
    Moreover, the $i$-th within-class variance for all $i\in [m]$ is determined as follows, based on Proposition~\ref{thm:app:emb:var:on:sphere},
    \begin{align*}
        \Var \big[\vU^{\delta}_i\big] 
        &=
        1- \frac{1}{n^2p^2} \norm{ \sum_{\vu^\delta\in\vU^{\delta}_i}\vu^\delta}_2^2
        =
        1- 
        \frac{1}{n^2p^2}
        \bigg(  n^2p^2 - \delta^2 \frac{mn^2(n-1)p^2}{mn-1} \bigg)
        =
        \delta^2 \frac{m(n-1)}{mn-1}.
    \end{align*}
    Therefore, the following holds for any $\delta\in \Big[0,\sqrt\frac{mn-1}{m(n-1)}\Big]$:
    \begin{align*}
        \Var^{\op{Btwn}} \big[\vU^{\delta}\big] + \sum_{i\in [m]}\frac{1}{m} \Var \big[\vU^{\delta}_i\big] 
        &=
        1 - \delta^2 \frac{m(n-1)}{mn-1}
        + \delta^2 \frac{m(n-1)}{mn-1}
        =1.
    \end{align*}
\end{proof}

\begin{theorem}[Existence of \ours{}]
    Suppose $mn\geq 2$ and $d\geq mn-1$ hold.
    Let a set of $mn$ vectors $\{\vw_{i,j}\}_{i\in[m], j\in[n]}$ forms the $(mn-1)$-simplex ETF in $\mathbb{R}^{d}$. For a given $\delta \in \Big[0,\sqrt\frac{mn-1}{m(n-1)}\Big]$,
    define the set of $mnp$ vectors $\vU^\delta := \big\{\vu_{i,j,k}^{\delta}\big\}_{i\in[m], j \in [n], k\in[p]}$ as
\begin{align}
    \label{eq:proof:u:vector:def}
    \vu_{i,j,k}^{\delta} := \delta \vw_{i,j} + h(\delta) \sum_{j\in[n]} \vw_{i,j}    \in \mathbb{R}^{d}
    \qquad\forall i\in[m], j \in [n], k \in [p],
\end{align}
where
\begin{align}
    \nonumber
    h(\delta) := -\frac{\delta}{n} \pm \frac{1}{n}\sqrt{\frac{\delta^2m(1-n) +(mn-1)}{m-1}}.
\end{align}
Then, the set of $mnp$ vectors $\vU^\delta$ constructs \ours{}.
\end{theorem}
\begin{proof}
From $d\geq mn-1$, the $(mn-1)$-simplex ETF of Def.~\ref{def:simplex} exists in $\mathbb{R}^{d}$, implying that the set of $mnp$ vectors $\vU^\delta$ exist. What remains to be proved is that $\vU^\delta$ follows \ours{} of Def.~\ref{def:model:general}.

For simplicity, the vector $\vv_{i}$ is defined as
\begin{align*}
    \vv_i := \sum_{j\in[n]} \vw_{i,j}    \in \mathbb{R}^{d}    
\end{align*}
for all $i\in[m]$. This simplifies $\vu_{i,j,k}^{\delta}$ as
\begin{align*}
    \vu_{i,j,k}^{\delta} = \delta \vw_{i,j} + h(\delta) \vv_i
    \qquad\forall i\in[m], j \in [n], k \in [p].  
\end{align*}

From the definition of the $(mn-1)$-simplex ETF, the followings hold for all $i \ne i' \in [m]$ and $j \in [n]$:
\begin{align}
    \vv_{i}^\top \vv_{i}
    &= n - n(n-1) \frac{1}{mn-1}
    \label{eq:proof:vandw:1} 
    =\frac{(m-1)n^2}{mn-1} \quad(>0),
    \\ 
    \vv_{i}^\top \vv_{i'}
    &= - n^2 \frac{1}{mn-1}
    \label{eq:proof:vandw:2} 
    = - \frac{n^2}{mn-1} \quad(<0),
    \\
    \vw_{i,j}^\top \vv_{i}
    &= 1 - (n-1) \frac{1}{mn-1}
    \label{eq:proof:vandw:3} 
    =\frac{(m-1)n}{mn-1} \quad(>0),
    \\
    \vw_{i,j}^\top \vv_{i'}
    &= - n \frac{1}{mn-1}
    \label{eq:proof:vandw:4} 
    = - \frac{n}{mn-1} \quad(<0).
\end{align}

Using the above results in \eqref{eq:proof:vandw:1}-\eqref{eq:proof:vandw:4}, we can show that $\vU^\delta$ defined in \eqref{eq:proof:u:vector:def} satisfies the conditions of \ours{}, which are \eqref{eq:model:same:general}, \eqref{eq:model:same:class:general}, and \eqref{eq:model:different:general} in Def.~\ref{def:model:general}, as below:

\paragraph{[Condition of \eqref{eq:model:same:general}]}
For all $i\in[m], j \in [n], k, k'\in[p]$,
\begin{align}
    (\vu_{i,j,k}^\delta)^\top \vu_{i,j,k'}^\delta 
    &= \label{eq:proof:abc:1}
    \delta^2 \vw_{i,j}^\top \vw_{i,j} + h(\delta)^2 \vv_{i}^\top \vv_{i}
    + 2 \delta h(\delta) \vw_{i,j}^\top \vv_{i}
    \\ &= \label{eq:proof:utoh}
    \delta^2
    + h(\delta)^2 \frac{(m-1)n^2}{mn-1}
    + 2 \delta h(\delta) \frac{(m-1)n}{mn-1}
    \\ &= \nonumber
    \delta^2
    + \frac{(m-1)n}{mn-1}
    \left(
        h(\delta)^2 \cdot n + 2 \delta h(\delta)
    \right)
    \\ &= \label{eq:proof:hsquared}
    \delta^2
    + \frac{(m-1)n}{mn-1} \cdot \frac{mn-1}{(m-1)n} \cdot (1-\delta^2)
    \\ &= 1, \nonumber
\end{align}
where the second and third terms in \eqref{eq:proof:utoh} are results from using \eqref{eq:proof:vandw:1} and \eqref{eq:proof:vandw:3}, respectively. Moreover, the second term in \eqref{eq:proof:hsquared} comes from
\begin{align*}
    h(\delta)^2 \cdot n
    &= 
    \left(-\frac{\delta}{n} \pm \frac{1}{n}\sqrt{\frac{\delta^2m(1-n) +(mn-1)}{m-1}}\right)^2
    \cdot n
    \\ &= 
    \left( \frac{\delta^2}{n^2} 
    + \frac{\delta^2m(1-n) +(mn-1)}{(m-1)n^2}
    \mp \frac{2\delta}{n^2}\sqrt{\frac{\delta^2m(1-n) +(mn-1)}{m-1}} \right)
    \cdot n
    \\ &= 
    \left(
    \frac{\delta^2(2m-mn-1) }{(m-1)n^2}
    + \frac{mn-1}{(m-1)n^2}
    \mp \frac{2\delta}{n^2}\sqrt{\frac{\delta^2m(1-n) +(mn-1)}{m-1}}
    \right)
    \cdot n
    \\ &= 
    \frac{\delta^2(2m-mn-1)}{(m-1)n}
    \mp \frac{2\delta}{n}\sqrt{\frac{\delta^2m(1-n) +(mn-1)}{m-1}}
    + \frac{mn-1}{(m-1)n}
    \\ &= 
    \frac{\delta^2(2m-mn-1)}{(m-1)n}
    - \frac{2\delta^2}{n}
    - 2 \delta h (\delta)
    + \frac{mn-1}{(m-1)n}
    \\ &= 
    - 2 \delta h (\delta)
    -\frac{\delta (mn-1)}{(m-1)n}
    + \frac{mn-1}{(m-1)n}
    \\ &= 
    - 2 \delta h (\delta)
    + \frac{mn-1}{(m-1)n} \cdot (1-\delta^2)
    .
\end{align*}

\paragraph{[Condition of \eqref{eq:model:same:class:general}]}
For all $i\in[m], j  \ne j'\in [n], k, k'\in[p]$,
\begin{align}
    (\vu_{i,j,k}^\delta)^\top \vu_{i,j',k'}^\delta 
    &= \nonumber
    \delta^2 \vw_{i,j}^\top \vw_{i,j'} + h(\delta)^2 \vv_{i}^\top \vv_{i}
    + \delta h(\delta) \vw_{i,j}^\top \vv_{i}  + \delta h(\delta) \vw_{i,j'}^\top \vv_{i}
    \\ &= \nonumber
    \delta^2 \vw_{i,j}^\top \vw_{i,j'} + h(\delta)^2 \vv_{i}^\top \vv_{i}
    + 2\delta h(\delta) \vw_{i,j}^\top \vv_{i}
    \\ &= \label{eq:proof:abc:2}
    \delta^2 \vw_{i,j}^\top \vw_{i,j'}
    +  (\vu_{i,j,k}^\delta)^\top \vu_{i,j,k'}^\delta  - \delta^2 \vw_{i,j}^\top \vw_{i,j} 
    \\ &= \label{eq:proof:abc:3}
    - \delta^2 \frac{1}{mn-1} + 1  - \delta^2 
    \\ &= \nonumber
    1 - \delta^2 \frac{mn}{mn-1},
\end{align}
where \eqref{eq:proof:abc:2} follows from \eqref{eq:proof:abc:1}, and the first and third terms in \eqref{eq:proof:abc:3} come form the definition of the ($mn-1$)-simplex ETF in Def.~\ref{def:simplex}.

\paragraph{[Condition of \eqref{eq:model:different:general}]}
For all $i \ne i'\in[m], j, j'\in [n], k, k'\in[p]$,
\begin{align}
    (\vu_{i,j,k}^\delta)^\top \vu_{i',j',k'}^\delta
    &= \nonumber
    \delta^2 \vw_{i,j}^\top \vw_{i',j'} + h(\delta)^2 \vv_{i}^\top \vv_{i'}
    + \delta h(\delta) \vw_{i,j}^\top \vv_{i'}
    + \delta h(\delta) \vw_{i',j'}^\top \vv_{i}
    \\ &= \nonumber
    \delta^2 \vw_{i,j}^\top \vw_{i',j'} + h(\delta)^2 \vv_{i}^\top \vv_{i'}
    + 2\delta h(\delta) \vw_{i,j}^\top \vv_{i'}
    \\ &= \label{eq:proof:abc:4}
    - \delta^2 \frac{1}{mn-1} 
    - h(\delta)^2 \frac{n^2}{mn-1}
    - 2\delta h(\delta) \frac{n}{mn-1} 
    \\ &= \nonumber
    - \delta^2 \frac{1}{mn-1} 
    - \frac{1}{m-1} \cdot \left(h(\delta)^2 \frac{(m-1)n^2}{mn-1}
    - 2\delta h(\delta) \frac{(m-1)n}{mn-1} \right)
    \\ &= \label{eq:proof:utoh:result}
    - \delta^2 \frac{1}{mn-1} 
    - \frac{1}{m-1} \cdot (1- \delta^2)  
    \\ &= \nonumber
    - \frac{1}{m-1}
    + \delta^2 \frac{m(n-1)}{(m-1)(mn-1)}
    ,
\end{align}
where the first term in \eqref{eq:proof:abc:4} comes from the definition of the the ($mn-1$)-simplex ETF in Def.~\ref{def:simplex}, and the second and third terms in \eqref{eq:proof:abc:4} comes from \eqref{eq:proof:vandw:2} and \eqref{eq:proof:vandw:4}, respectively. Moreover, \eqref{eq:proof:utoh:result} comes from reformatting \eqref{eq:proof:utoh} which is equal to $1$.
\end{proof}

\newpage
\subsection{The Optimality of \ours{}}
\label{appendix:proof:ssem}

First, we prove one proposition and two lemmas that are necessary for proving the theorem, which shows the optimality of \ours{}.

\begin{proposition}
    \label{thm:multivariate:jensen}
    Let $\va\in\mathbb{R}^d$ be a given vector where all elements are non-negative. Then, the following holds for any set of $n$ vectors $\{\vw_i \in \mathbb{R}^d\}_{i\in[n]}$:
    \begin{equation}
        \label{eq:multivariate:jensen}
        \frac{1}{n}\sum_{i \in [n]}\log \bigg( \va^\top \exp(\vw_i) \bigg)
        \geq
        \log \bigg(\va^\top \exp\bigg(\frac{1}{n}\sum_{i \in [n]}\vw_i\bigg)\bigg),
    \end{equation}
    where $\exp$ is an element-wise exponential function.
    
    When the given vector $\va$ has all positive elements, the equality condition of \eqref{eq:multivariate:jensen} is 
    \begin{equation}
        \nonumber
        \vw_i = \vc \qquad \forall i \in [n],
    \end{equation}
    for some $\vc\in\mathbb{R}^d$.
\end{proposition}
\begin{proof}
    Note that the both terms inside the logarithm in \eqref{eq:multivariate:jensen} equal zero when $\va =\bm{0}\in\mathbb{R}^d$, \ie
    \begin{equation*}
    \va^\top \exp(\vw_i) = 0 =\va^\top \exp\bigg(\frac{1}{n}\sum_{i \in [n]}\vw_i\bigg).
    \end{equation*}
    Therefore, it suffices to consider the case where $\va$ contains only positive elements.
    
    Let the continuous function $f(\vw)= \log \left( \va^\top \exp(\vw) \right)$ be defined for the given vector $\va$. Then, for any $\vw,\vv \in\mathbb{R}^d$, the followings hold:
    \begin{align}
        \frac{1}{2}f(\vw) + \left(1-\frac{1}{2}\right)f(\vv)
        &= \nonumber
        \frac{1}{2} \left(\log \left( \va^\top \exp(\vw) \right) + \log \left( \va^\top \exp(\vv) \right)\right)
        \\ &= \nonumber
        \log\sqrt{\va^\top \exp(\vw) \cdot \va^\top \exp(\vv)}
        \\ &\geq \label{eq:multivariate:jensen:proof:equality}
        \log\left(\va^\top \exp\left(\frac{1}{2}\vw+\frac{1}{2}\vv\right) \right)
        \\ & = \nonumber
        f\left(\frac{1}{2}\vw + \left(1-\frac{1}{2}\right) \vv\right),
    \end{align}
    where the equality condition of \eqref{eq:multivariate:jensen:proof:equality} is $\exp\left(\frac{1}{2}\vw\right) = \exp\left(\frac{1}{2}\vv\right)$ from the Cauchy–Schwarz inequality. This directly implies that $f$ is a convex function due to the continuity of $f$. 

    As a result, \eqref{eq:multivariate:jensen} holds from the Jensen's inequality as below:
    \begin{align*}
        \frac{1}{n}\sum_{i \in [n]}\log \bigg( \va^\top \exp(\vw_i) \bigg)
        &=
        \frac{1}{n}\sum_{i \in [n]} f(\vw_i)
        \geq
        f\bigg(\frac{1}{n}\sum_{i \in [n]}\vw_i\bigg)
        =
        \log \bigg(\va^\top \exp\bigg(\frac{1}{n}\sum_{i \in [n]}\vw_i\bigg)\bigg),
    \end{align*}
    where the equality condition of \eqref{eq:multivariate:jensen:proof:equality} is simplified as 
    \begin{equation*}
        \vw_i = \vc \qquad \forall i \in [n],
    \end{equation*}
    for some $\vc\in\mathbb{R}^d$.
\end{proof}

\begin{lemma}
    \label{thm:lemma:1}
    Let $\vU := \{\vu_{i,j,k}\}_{i\in[m], j \in [n], k \in [p]}$ be a set of $mnp$ vectors in $\sR^d$, satisfying $\|\vu\|_2^2=1$ for all $\vu \in \vU$. Additionally, define the sets $\vU_{i,j}:=\{\vu_{i,j,k}\}_{k\in[p]}$ and $\vU_{i}:=\cup_{j\in[n]}\vU_{i,j}$ for all $i\in[m]$ and $j\in[n]$. 
    Then, for every constant $c \in[-mn,mn(n-1)]$, 
    there exists a unique $\delta^\star (c) \in \Big[0,\sqrt\frac{mn-1}{m(n-1)}\Big]$ such that 
    \begin{equation*}
     \vU^{\delta^\star(c)}
     \in
     \argmin_{\vU}
     \Bigg\{
     \sum_{i\in[m]}\sum_{\vu \in \vU_{i}}\sum_{\vv \in \vU \setminus \vU_i}\vu^\top \vv
     \Bigg|
     \sum_{i\in[m],j \ne j'\in[n]} \E [\vU_{i,j}]^\top\E[ \vU_{i,j'}] =c 
     \Bigg\}.
    \end{equation*}
    Specifically, $\delta^\star(c)=\sqrt{ \frac{mn-1}{mn} - \frac{mn-1}{m^2n^2(n-1)} c }$.
\end{lemma}

\begin{proof}
First of all, the possible values of $c$ are bounded within the interval $[-mn,mn(n-1)]$: the maximum value of $c$ can be achieved when $\vu=\vv$ for all $\vu, \vv \in \vU$, while the minimum is determined as below.
\begin{align}
    c &=
    \sum_{i\in[m],j \ne j'\in[n]} \E [\vU_{i,j}]^\top\E[ \vU_{i,j'}]
    \\ &= \nonumber
    \frac{1}{p^2}\sum_{i\in[m],j \ne j'\in[n]} \sum_{\vu \in \vU_{i,j}}\sum_{\vv \in \vU_{i,j'}}\vu^\top \vv
     \\ &= \nonumber
    \frac{1}{p^2}\sum_{i\in[m]} \sum_{\vu \in \vU_{i}}\sum_{\vv \in \vU_{i}}\vu^\top \vv 
    - \frac{1}{p^2}\sum_{i\in[m],j\in[n]} \sum_{\vu \in \vU_{i,j}}\sum_{\vv \in \vU_{i,j}}\vu^\top \vv 
    \\&= \nonumber
    \frac{1}{p^2}\sum_{i\in[m]} \norm{\sum_{\vu \in \vU_{i}}\vu}_2^2
    - \frac{1}{p^2}\sum_{i\in[m],j\in[n]} \norm{\sum_{\vu \in \vU_{i,j}}\vu}_2^2
    \\& \geq \nonumber
    0- mn
\end{align}

Now consider the minimization problem. Given that $\|\vu\|_2^2=1$ for all $\vu \in \vU$, the following holds:
\begin{align*}
    \left\|\sum_{\vu \in \vU} \vu \right\|_2^2
    &= 
    \sum_{\vu \in \vU}\sum_{\vv \in \vU} \vu^\top \vv
    \\
    &=
    \sum_{i\in[m],j \in[n]} \sum_{\vu \in \vU_{i,j}}\sum_{\vv \in \vU_{i,j}}\vu^\top \vv
    + \sum_{i\in[m],j \ne j'\in[n]} \sum_{\vu \in \vU_{i,j}}\sum_{\vv \in \vU_{i,j'}}\vu^\top \vv
    + \sum_{i\in[m]}\sum_{\vu \in \vU_{i}}\sum_{\vv \in \vU \setminus \vU_i}\vu^\top \vv,
\end{align*}
which implies
\begin{align}
    \sum_{i\in[m]}\sum_{\vu \in \vU_{i}}\sum_{\vv \in \vU \setminus \vU_i}\vu^\top \vv
    &= \nonumber
    \left\|\sum_{\vu \in \vU} \vu \right\|_2^2
    - \sum_{i\in[m],j \in[n]} \sum_{\vu \in \vU_{i,j}}\sum_{\vv \in \vU_{i,j}}\vu^\top \vv
    - \sum_{i\in[m],j \ne j'\in[n]} \sum_{\vu \in \vU_{i,j}}\sum_{\vv \in \vU_{i,j'}}\vu^\top \vv
    \\ \label{eq:proof:lemma:inequality:1}
    & \geq
    0
    - \sum_{i\in[m],j \in[n]} \sum_{\vu \in \vU_{i,j}}\sum_{\vv \in \vU_{i,j}} 1
    - \sum_{i\in[m],j \ne j'\in[n]} \sum_{\vu \in \vU_{i,j}}\sum_{\vv \in \vU_{i,j'}}\vu^\top \vv 
    \\ \label{eq:proof:lemma:argmin:given:c:result}
    & =
    - mn p^2
    - p^2 c .
\end{align}
Note that the equality conditions of \eqref{eq:proof:lemma:inequality:1} are
\begin{align}
    \label{eq:proof:lemma:argmin:given:c:condition:centroid}
    \sum_{\vu \in \vU} \vu &= 0,
    \\
    \label{eq:proof:lemma:argmin:given:c:condition:instance}
    \vu^\top \vv &= 1
    &\forall 
    i\in[m],j \in[n], \vu \in \vU_{i,j}, \vv \in \vU_{i,j}
    ,
\end{align}
implying that the centroid of the embedding vectors is at the origin and that every embedding vector of each instance is the same, regardless of augmentations.

For any $\delta \in \Big[0,\sqrt\frac{mn-1}{m(n-1)}\Big]$, embedding vectors $\vU^\delta$ of \ours{} in Def.~\ref{def:model:general}  fulfill the equality conditions in \eqref{eq:proof:lemma:argmin:given:c:condition:centroid} and \eqref{eq:proof:lemma:argmin:given:c:condition:instance}, as follows:
\begin{align*}
    \left\|\sum_{\vu^\delta \in \vU^\delta} \vu^\delta \right\|_2^2
    &= 
    \sum_{i\in[m],j \in[n]} \sum_{\vu^\delta \in \vU_{i,j}^\delta}\sum_{\vv^\delta \in \vU_{i,j}^\delta}\left(\vu^\delta\right)^\top \vv^\delta
    + \sum_{i\in[m],j \ne j'\in[n]} \sum_{\vu^\delta \in \vU_{i,j}^\delta}\sum_{\vv^\delta \in \vU_{i,j'}^\delta}\left(\vu^\delta\right)^\top \vv^\delta
    \\&\quad
    + \sum_{i\in[m]}\sum_{\vu^\delta \in \vU_{i}^\delta}\sum_{\vv^\delta \in \vU^\delta \setminus \vU_i^\delta}\left(\vu^\delta\right)^\top \vv^\delta
    \\
    &=
    mnp^2
    + mn(n-1)p^2 \bigg(  1 - \delta^2 \frac{mn}{mn-1} \bigg)
    + m(m-1)n^2p^2 \bigg( - \frac{1}{m-1} + \delta^2 \frac{m(n-1)}{(m-1)(mn-1)} \bigg)
    \\
    &= 0,
\end{align*}
implying \eqref{eq:proof:lemma:argmin:given:c:condition:centroid} holds, and \eqref{eq:proof:lemma:argmin:given:c:condition:instance} holds from \eqref{eq:model:same:general} in Def~\ref{def:model:general}.

Therefore, we can conclude the existence of a unique $\delta\in\Big[0,\sqrt\frac{mn-1}{m(n-1)}\Big]$ such that $\vU^{\delta}$ of \ours{} represents the optimal embedding from \eqref{eq:proof:lemma:argmin:given:c:result}, specified as:
\begin{align*}
    - mnp^2 - p^2
    &=
    \sum_{i\in[m]}\sum_{\vu^\delta \in \vU_{i}^\delta}\sum_{\vv^\delta \in \vU^\delta \setminus \vU_i^\delta}\left(\vu^\delta\right)^\top \vv^\delta
    \\ &=
    m(m-1)n^2p^2
    \bigg( - \frac{1}{m-1} + \delta^2 \frac{m(n-1)}{(m-1)(mn-1)}\bigg)
    \\ &=
    -mn^2p^2
    + \delta^2 \frac{m^2n^2(n-1)p^2}{mn-1},
\end{align*}
which is equal to
\begin{align}
    \delta
    &= \nonumber
    \sqrt{ \frac{mn-1}{m^2n^2(n-1)p^2} \Big(mn(n-1)p^2 - p^2\Big) }
    \\ &= \label{eq:ssem:lemma:delta:function}
    \sqrt{ \frac{mn-1}{mn} - \frac{mn-1}{m^2n^2(n-1)} c }
    .
\end{align}
The uniqueness of $\delta$ comes from the fact that \eqref{eq:ssem:lemma:delta:function} is a strictly decreasing function of $c\in[-mn,mn(n-1)]$.
\end{proof}

\begin{lemma}
    \label{thm:lemma:2}
    Let $\vU := \{\vu_{i,j,k}\}_{i\in[m], j \in [n], k \in [p]}$ be a set of $mnp$ vectors in $\sR^d$, satisfying $\|\vu\|_2^2=1$ for all $\vu \in \vU$. Additionally, define the sets $\vU_{i,j}:=\{\vu_{i,j,k}\}_{k\in[p]}$ and $\vU_{i}:=\cup_{j\in[n]}\vU_{i,j}$ for all $i\in[m]$ and $j\in[n]$. 
    Then, for every constant $c \in[0,2mn^2]$, 
    there exists a unique $\delta^\star (c) \in \Big[0,\sqrt{\frac{mn-1}{mn(n-1)}}\Big]$ such that 
    \begin{equation}
    \label{eq:lemma:argmax:2}
     \vU^{\delta^\star(c)}
     \in
     \argmax_{\vU}
     \Bigg\{
     \sum_{i\in[m],j \ne j'\in[n]} \E [\vU_{i,j}]^\top\E [\vU_{i,j'}]
     \Bigg|
     \sum_{i\in[m],j \ne j'\in[n]} \big\|\E[\vU_{i,j}] - \E[\vU_{i,j'}] \big\|_2^2  =c 
     \Bigg\}.
    \end{equation}
    Specifically, $\delta^\star(c)=\sqrt{\frac{mn-1}{2m^2n^2(n-1)} \cdot c}$.
\end{lemma}
\begin{proof}
First of all, the possible values of $c$ are bounded within the interval $[0,2mn^2]$, because $\|\vu\|_2^2=1$ for all $\vu \in \vU$. Especially, the possible maximum value of $c$, which is $2mn^2$, can be attained as below.
\begin{align}
    c &= \nonumber
    \sum_{i\in[m],j \ne j'\in[n]} \big\|\E[\vU_{i,j}] - \E[\vU_{i,j'}] \big\|_2^2 
     \\ &= \nonumber
    \sum_{i\in[m],j \ne j'\in[n]} \bigg(
        \big\|\E[\vU_{i,j}]\big\|_2^2 
        + \big\|\E[\vU_{i,j'}]\big\|_2^2 
        -2\E[\vU_{i,j}]^\top\E[\vU_{i,j'}]
    \bigg) 
    \\&= \label{eq:lemma:argmax:2:for:proof:1}
     2(n-1) \sum_{i\in[m],j\in[n]} \big\|\E[\vU_{i,j}] \big\|_2^2 
     -2 \sum_{i\in[m],j\ne j'\in[n]} \E[\vU_{i,j}]^\top \E[\vU_{i,j'}]
    \\&= \nonumber
     2(n-1) \sum_{i\in[m],j\in[n]} \big\|\E[\vU_{i,j}] \big\|_2^2 
     -2 \Bigg( \sum_{i\in[m]} \bigg\|
        \sum_{j\in[n]} \E[\vU_{i,j}]
       \bigg\|_2^2
        - \sum_{i\in[m],j\in[n]} \big\|\E[\vU_{i,j}] \big\|_2^2 
        \Bigg)
    \\&= \nonumber
     2n\sum_{i\in[m],j\in[n]} \big\|\E[\vU_{i,j}] \big\|_2^2 
     -2 \sum_{i\in[m]} \bigg\|
        \sum_{j\in[n]} \E[\vU_{i,j}]
       \bigg\|_2^2
    \\ &\leq \nonumber
     2n\sum_{i\in[m],j\in[n]} 1
     -0
    \\ &= \nonumber
    2mn^2,
\end{align}
where $\big\|\E[\vU_{i,j}] \big\|_2^2 = \bigg\|\frac{1}{p}\sum_{\vu \in\vU_{i,j}} \vu  \bigg\|_2^2  \leq \frac{1}{p}\sum_{\vu \in\vU_{i,j}} \big\| \vu  \big\|_2^2 = 1$ from Jensen's inequality.

Now, consider the main maximization problem as given in \eqref{eq:lemma:argmax:2}. From \eqref{eq:lemma:argmax:2:for:proof:1},
\begin{align}
    c &= \nonumber
     2(n-1) \sum_{i\in[m],j\in[n]} \big\|\E[\vU_{i,j}] \big\|_2^2 
     -2 \sum_{i\in[m],j\ne j'\in[n]} \E[\vU_{i,j}]^\top \E[\vU_{i,j'}]
    \\& \leq \label{eq:proof:lemma:condition:2}
     2(n-1) \sum_{i\in[m],j\in[n]} 1
     -2 \sum_{i\in[m],j\ne j'\in[n]} \E[\vU_{i,j}]^\top \E[\vU_{i,j'}]
    \\& = \nonumber
     2mn(n-1) 
     -2 \sum_{i\in[m],j\ne j'\in[n]} \E[\vU_{i,j}]^\top \E[\vU_{i,j'}]
\end{align}
which implies
\begin{align}
     \sum_{i\in[m],j\ne j'\in[n]} \E[\vU_{i,j}]^\top \E[\vU_{i,j'}]
     &\leq \label{eq:proof:lemma:result:2}
     mn(n-1)
     - \frac{c}{2}.
\end{align}

The equality condition of \eqref{eq:proof:lemma:condition:2} is
\begin{equation*}
    \label{eq:proof:lemma:condition:3}
    \big\|\E[\vU_{i,j}]\big\|_2^2 = 1
    \qquad \forall 
    i\in[m],j \in[n]
    ,    
\end{equation*}
implying that every embedding vector of each instance is the same, regardless of augmentation.
For any $\delta \in \Big[0,\sqrt\frac{mn-1}{m(n-1)}\Big]$, the set of embedding vectors $\vU^\delta$ of \ours{} fulfills the equality condition in \eqref{eq:proof:lemma:condition:2} from \eqref{eq:model:same:general} in Def~\ref{def:model:general} as follows:
\begin{align*}
    \bigg\|\E\big[\vU_{i,j}^\delta\big]\bigg\|_2^2
    &=
    \Bigg\|\frac{1}{p}\sum_{\vu^\delta\in\vU_{i,j}^\delta} \vu^\delta \Bigg\|_2^2
    =
    \frac{1}{p^2}\sum_{\vu^\delta\in\vU_{i,j}^\delta} \sum_{\vv^\delta\in\vU_{i,j}^\delta}\left(\vu^\delta\right)^\top \vv^\delta
    = 1
    &\forall  i\in[m],j \in[n].
\end{align*}

Therefore, we can conclude the existence of a unique $\delta\in\Big[0,\sqrt\frac{mn-1}{m(n-1)}\Big]$ such that $\vU^{\delta}$ of \ours{} represents the optimal embedding from \eqref{eq:proof:lemma:result:2}, specified as:
\begin{align*}
     mn(n-1)p^2 - \frac{cp^2}{2}
     &= 
     p^2 \sum_{i\in[m],j\ne j'\in[n]} \E[\vU_{i,j}]^\top \E[\vU_{i,j'}]
    \\ &=
    \sum_{i\in[m],j \ne j'\in[n]} \sum_{\vu^\delta \in \vU_{i,j}^\delta}\sum_{\vv^\delta \in \vU_{i,j'}^\delta}\left(\vu^\delta\right)^\top \vv^\delta
    \\ &=
    mn(n-1)p^2
    \bigg( 1 - \delta^2 \frac{mn}{mn-1}\bigg)
    \\ &=
    mn(n-1)p^2
    - \delta^2 \frac{m^2n^2(n-1)p^2}{mn-1},
\end{align*}
which is equal to
\begin{align}
    \delta
    &= \label{eq:ssem:lemma:delta:function:2}
    \sqrt{\frac{mn-1}{2m^2n^2(n-1)} \cdot c}
    \quad .
\end{align}
The uniqueness of $\delta$ comes from the fact that \eqref{eq:ssem:lemma:delta:function:2} is a strictly increasing function of $c\in[0,2mn^2]$.
\end{proof}

\begin{theorem}[Optimality of \ours{}]
    \label{thm:app:optimal:ssem}
    Suppose $mn\geq 2$ and $d\geq mn-1$ hold.
    Then, all embedding sets $\vU^{\star}$ that minimize the loss $\loss(\vU)$ in (1) are included in the \ours{} in Def.~\ref{def:model}, \ie
    \begin{align}
        \nonumber
        \forall \vU^{\star} \in \argmin_{\vU} \loss(\vU), \quad \exists! \delta \in [0,1] \text{ such that } \vU^{\delta} = \vU^{\star}.
    \end{align}
    Specifically,
    \begin{equation}
        \label{eq:proof:find:optimal:delta}
        \delta^\star = 
        \begin{cases}
            0,
            &\text{if } h\big(0; m,n,\tau,\alpha\big)\geq0,
            \\
            \delta\in \left( 0, 1\right]
            \text{ such that }h\left(\delta^2\frac{mn}{mn-1}; m,n,\tau,\alpha\right)=0,
            &\text{otherwise},
        \end{cases}
    \end{equation}
    where 
    \begin{align*}
    h\big(x; m,n,\tau,\alpha\big)
    &=
    (1-\alpha)
    - \alpha(n-1)\cdot\exp( - x /\tau)
    \\
    &\quad
    + (mn-1-\alpha (m-1)n) \cdot\exp\left(\left(- \frac{m}{m-1} + x \frac{n-1}{(m-1)n}\right) / \tau\right).    
    \end{align*}
\end{theorem}

\begin{proof}
We want to find the optimal embeddings that minimize a SupCL loss $\loss(\vU)$ as defined in \eqref{eq:loss}, which is a convex combination of $\loss_{\op{Sup}}(\vU)$ in \eqref{eq:loss:sup:assume} and $\loss_{\op{Self}}(\vU)$ in \eqref{eq:loss:cl:assume}. These can be rewritten as follows:
\begin{align}
\loss_{\op{Sup}}(\vU) 
& =  \nonumber
-\frac{1}{mn(n-1)p^2}\sum_{i\in [m], j \ne j'\in[n]} \sum_{\vu \in \vU_{i,j}}\sum_{\vv \in \vU_{i,j'}}
\log \frac{\exp(\vu^\top \vv / \tau)}{\sum_{\vw \in \vU}\exp(\vu^\top \vw / \tau)}
\\
& =  \nonumber
\frac{1}{mn(n-1)p^2}\sum_{i\in[m], j\ne j'\in[n]}  \sum_{\vu \in \vU_{i,j}}\sum_{\vv \in \vU_{i,j'}}
\log\left(\sum_{\vw \in \vU}\exp(\vu^\top (\vw-\vv) / \tau)\right)
\\
& =  \label{eq:proof:main:loss:assume:sup}
\frac{1}{mn(n-1)p^2}\sum_{i\in[m], j\in[n]}  \sum_{\vu \in \vU_{i,j}}\sum_{\vv \in \vU_{i} \setminus \vU_{i,j}}
\log\left(\sum_{\vw \in \vU}\exp(\vu^\top (\vw-\vv) / \tau)\right),
\\ 
\loss_{\op{Self}}(\vU)
&=  \nonumber
-\frac{1}{mnp^2}\sum_{i\in [m], j\in[n]} \sum_{\vu \in \vU_{i,j}} \sum_{\vv \in \vU_{i,j}} \log \frac{\exp(\vu^\top \vv / \tau)}{\sum_{\vw\in\vU}\exp(\vu^\top \vw / \tau)}
\\ &=  \label{eq:proof:main:loss:assume:cl}
\frac{1}{mnp^2}\sum_{i\in [m], j\in[n]} \sum_{\vu \in \vU_{i,j}}\sum_{\vv \in \vU_{i,j}}
\log\left(\sum_{\vw \in \vU}\exp(\vu^\top (\vw-\vv) / \tau)\right).
\end{align}

We first consider minimizing $\loss_{\op{Sup}}(\vU)$ in \eqref{eq:proof:main:loss:assume:sup}, and then minimizing $\loss_{\op{Self}}(\vU)$ in \eqref{eq:proof:main:loss:assume:cl} in a similar manner.

Note that the set of all embeddings $\vU$ can be partitioned to three disjoint sets as
\begin{equation*}
    \vU = (\vU \setminus \vU_{i}) \; \dot{\cup}\; (\vU_i \setminus \vU_{i,j}) \; \dot{\cup}\; \vU_{i,j}
    \qquad
    \forall i \in [m], j \in[n],
\end{equation*}
where $\dot{\cup}$ denotes the disjoint union. Then, the term inside the logarithm in \eqref{eq:proof:main:loss:assume:sup} can be decomposed to three terms. Specifically, for all $i \in [m], j\ne j' \in [n]$, $\vu \in \vU_{i,j}$, and $\vv \in \vU_{i,j'}$, 

\begin{align}
\sum_{\vw \in \vU} \exp(\vu^\top (\vw-\vv) / \tau)
& =  \nonumber
    \sum_{\vw \in \vU \setminus \vU_{i}}
    \exp(\vu^\top (\vw-\vv) / \tau)
    + \sum_{\vw \in \vU_i \setminus \vU_{i,j}}\exp(\vu^\top (\vw-\vv) / \tau)
    + \sum_{\vw\in\vU_{i,j}}\exp(\vu^\top (\vw-\vv) / \tau)
    \\
    &\geq \nonumber
    (m-1)np
    \cdot
    \exp\left(\frac{1}{(m-1)np}\sum_{\vw \in \vU \setminus \vU_i}\vu^\top \vw / \tau - \vu^\top \vv / \tau\right)
    \\& \nonumber
    \qquad + (n-1)p\cdot\exp\left(\frac{1}{(n-1)p}\sum_{\vw \in \vU_i \setminus \vU_{i,j}}\vu^\top \vw / \tau - \vu^\top\vv / \tau\right)
    \\& \label{eq:proof:main:ineq:1}
    \qquad + p\cdot\exp\left(\frac{1}{p}\sum_{\vw \in \vU_{i,j}}\vu^\top \vw/\tau - \vu^\top \vv / \tau\right)
    ,
\end{align}
where the inequality in \eqref{eq:proof:main:ineq:1} comes from using Jensen's inequality three times. The equality in \eqref{eq:proof:main:ineq:1} is achieved if there exist some constants $c_1, c_2, c_3 \in \sR$ such that the following conditions hold for all $i \in [m], j\ne j' \in [n]$, $\vu \in \vU_{i,j}$, and $\vv \in \vU_{i,j'}$:
\begin{align}
    \vu^\top (\vw-\vv)  &= c_{1}
    \qquad\qquad \forall \vw \in \vU \setminus \vU_{i},
    \label{eq:proof:main:condition:1}
    \\
    \vu^\top (\vw-\vv) &= c_{2}
    \qquad\qquad \forall \vw \in \vU_i\setminus \vU_{i,j},
    \label{eq:proof:main:condition:2}
    \\
    \vu^\top (\vw-\vv) &= c_{3}
    \qquad\qquad \forall \vw \in \vU_{i,j}.
    \label{eq:proof:main:condition:3}
\end{align}

By using the above result in \eqref{eq:proof:main:ineq:1} to $\loss_{\op{Sup}}(\vU)$ in \eqref{eq:proof:main:loss:assume:sup}, 
\begin{align}
\loss_{\op{Sup}}(\vU)
& \geq \nonumber 
\frac{1}{mn(n-1)p^2}\sum_{i\in[m], j\in[n]}  \sum_{\vu \in \vU_{i,j}}\sum_{\vv \in \vU_{i} \setminus \vU_{i,j}}
    \log 
    \Bigg( 
    (m-1)np\cdot
    \exp\bigg(\frac{1}{(m-1)np}\sum_{\vw \in \vU \setminus \vU_i}\vu^\top \vw / \tau - \vu^\top \vv / \tau\bigg)
    \\&\qquad\qquad\qquad\qquad\qquad\qquad\qquad\qquad\qquad\qquad\;  \nonumber 
    + (n-1)p\cdot\exp\bigg(\frac{1}{(n-1)p}\sum_{\vw \in \vU_i \setminus \vU_{i,j}}\vu^\top \vw / \tau - \vu^\top\vv / \tau\bigg)
    \\&\qquad\qquad\qquad\qquad\qquad\qquad\qquad\qquad\qquad\qquad\;  \nonumber 
    + p\cdot\exp\bigg(\frac{1}{p}\sum_{\vw \in \vU_{i,j}}\vu^\top \vw/\tau - \vu^\top \vv / \tau\bigg)
    \Bigg)
\\& \geq\nonumber 
    \frac{1}{m}\sum_{i\in[m]}
    \log 
    \Bigg(
    (m-1)np\cdot
    \exp\bigg(
         \frac{1}{n(n-1)p^2}\sum_{j\in[n]}  \sum_{\vu \in \vU_{i,j}}\sum_{\vv \in \vU_{i} \setminus \vU_{i,j}}
         \frac{1}{(m-1)np}\sum_{\vw \in \vU \setminus \vU_i}\vu^\top \vw / \tau
         \\&\qquad\qquad\qquad\qquad\qquad\qquad\qquad\qquad \nonumber
        -\frac{1}{n(n-1)p^2}\sum_{j\in[n]}  \sum_{\vu \in \vU_{i,j}}\sum_{\vv \in \vU_{i} \setminus \vU_{i,j}}
        \vu^\top \vv / \tau
    \bigg)
    \\&\qquad\qquad\qquad\qquad \nonumber 
    + (n-1)p\cdot\exp\bigg(
        \frac{1}{n(n-1)p^2}\sum_{j\in[n]}  \sum_{\vu \in \vU_{i,j}}\sum_{\vv \in \vU_{i} \setminus \vU_{i,j}}\frac{1}{(n-1)p}\sum_{\vw \in \vU_i \setminus \vU_{i,j}}\vu^\top \vw / \tau 
        \\&\qquad\qquad\qquad\qquad\qquad\qquad\qquad\qquad \nonumber
        - \frac{1}{n(n-1)p^2}\sum_{j\in[n]}  \sum_{\vu \in \vU_{i,j}}\sum_{\vv \in \vU_{i} \setminus \vU_{i,j}}\vu^\top\vv / \tau
    \bigg)
    \\&\qquad\qquad\qquad\qquad \nonumber
    + p\cdot\exp\bigg(
        \frac{1}{n(n-1)p^2}\sum_{j\in[n]}  \sum_{\vu \in \vU_{i,j}}\sum_{\vv \in \vU_{i} \setminus \vU_{i,j}}\frac{1}{p}\sum_{\vw \in \vU_{i,j}}\vu^\top \vw/\tau
        \\&\qquad\qquad\qquad\qquad\qquad\qquad\quad \label{eq:proof:main:ineq:2}
        - \frac{1}{n(n-1)p^2}\sum_{j\in[n]}  \sum_{\vu \in \vU_{i,j}}\sum_{\vv \in \vU_{i} \setminus \vU_{i,j}}\vu^\top \vv / \tau
    \bigg)
\Bigg)
\\
& =\nonumber 
    \frac{1}{m}\sum_{i\in[m]}
    \log 
    \Bigg(
    (m-1)np
    \cdot\exp\bigg(
         \frac{1}{(m-1)n^2p^2} \sum_{\vu \in \vU_{i}}\sum_{\vw \in \vU \setminus \vU_i}\vu^\top \vw / \tau
        \\&\qquad\qquad\qquad\qquad\qquad\qquad\qquad\quad \nonumber
        - \frac{1}{n(n-1)p^2}\sum_{j\in[n]}  \sum_{\vu \in \vU_{i,j}}\sum_{\vv \in \vU_{i} \setminus \vU_{i,j}}\vu^\top \vv / \tau
    \bigg)
    \\&\qquad\qquad\qquad\qquad \nonumber 
    + (n-1)p\cdot\exp\bigg(
        \frac{1}{n(n-1)p^2}\sum_{j \in[n]} \sum_{\vu \in \vU_{i,j}}\sum_{\vw \in \vU_i \setminus \vU_{i,j}}\vu^\top \vw / \tau 
        \\&\qquad\qquad\qquad\qquad\qquad\qquad\qquad\quad\quad \nonumber
        - \frac{1}{n(n-1)p^2}\sum_{j\in[n]}  \sum_{\vu \in \vU_{i,j}}\sum_{\vv \in \vU_{i} \setminus \vU_{i,j}}\vu^\top\vv / \tau
    \bigg)
    \\&\qquad\qquad\qquad\qquad \label{eq:proof:main:result:new:1}
    + p\cdot\exp\bigg(
        \frac{1}{np^2}\sum_{j \in[n]} \sum_{\vu \in \vU_{i,j}}\sum_{\vw \in \vU_{i,j}}\vu^\top \vw/\tau
        - \frac{1}{n(n-1)p^2}\sum_{j\in[n]}  \sum_{\vu \in \vU_{i,j}}\sum_{\vv \in \vU_{i} \setminus \vU_{i,j}}\vu^\top \vv / \tau
    \bigg)
\Bigg)
\end{align}
where the inequality in \eqref{eq:proof:main:ineq:2} holds from Proposition~\ref{thm:multivariate:jensen}. Note that the value inside the second exponential in \eqref{eq:proof:main:result:new:1} is zero.

Moreover, applying Proposition~\ref{thm:multivariate:jensen} one more to \eqref{eq:proof:main:result:new:1} results in
\begin{align}
    \loss_{\op{Sup}}(\vU)
    & \geq\nonumber 
    \log 
    \Bigg(
    (m-1)np\cdot
    \exp\bigg(
         \frac{1}{m(m-1)n^2p^2}\sum_{i\in[m]}\sum_{\vu \in \vU_{i}}\sum_{\vw \in \vU \setminus \vU_i}\vu^\top \vw / \tau
        \\&\qquad\qquad\qquad\qquad\qquad\quad \nonumber
        - \frac{1}{mn(n-1)p^2}\sumi\sum_{j\in[n]}  \sum_{\vu \in \vU_{i,j}}\sum_{\vv \in \vU_{i} \setminus \vU_{i,j}}\vu^\top \vv / \tau
    \bigg)
    \\&\qquad\qquad \nonumber 
    + (n-1)p
    \\&\qquad\qquad \nonumber 
    + p\cdot\exp\bigg(
        \frac{1}{mnp^2}\sum_{i\in[m],j \in[n]} \sum_{\vu \in \vU_{i,j}}\sum_{\vw \in \vU_{i,j}}\vu^\top \vw/\tau
    \\&\qquad\qquad\qquad\qquad\quad \label{eq:proof:main:ineq:3}
        - \frac{1}{mn(n-1)p^2}\sum_{i\in[m],j\in[n]}  \sum_{\vu \in \vU_{i,j}}\sum_{\vv \in \vU_{i} \setminus \vU_{i,j}}\vu^\top \vv / \tau
    \bigg)
    \Bigg).
\end{align}
The equality conditions of \eqref{eq:proof:main:ineq:2} and \eqref{eq:proof:main:ineq:3}, which use Proposition~\ref{thm:multivariate:jensen}, are also achieved if the conditions in \eqref{eq:proof:main:condition:1}, \eqref{eq:proof:main:condition:2}, and \eqref{eq:proof:main:condition:3} are satisfied.

Note that the value inside the first exponential in \eqref{eq:proof:main:ineq:3} follows the below inequality:
\begin{align}
    &\nonumber
    \frac{1}{m(m-1)n^2p^2}\sum_{i\in[m]}\sum_{\vu \in \vU_{i}}\sum_{\vw \in \vU \setminus \vU_i}\vu^\top \vw / \tau
    - \frac{1}{mn(n-1)p^2}\sumi\sum_{j\in[n]}  \sum_{\vu \in \vU_{i,j}}\sum_{\vv \in \vU_{i} \setminus \vU_{i,j}}\vu^\top \vv / \tau
    \\ 
    &\qquad \geq \nonumber
    -\frac{mnp^2}{m(m-1)n^2p^2} /\tau
    - \frac{1}{m(m-1)n^2p^2}\sum_{i\in[m],j \ne j'\in[n]} \sum_{\vu \in \vU_{i,j}}\sum_{\vv \in \vU_{i,j'}}\vu^\top \vv / \tau
    \\&\qquad \qquad \label{eq:proof:main:ineq:4}
    - \frac{1}{mn(n-1)p^2}\sumi\sum_{j\in[n]}  \sum_{\vu \in \vU_{i,j}}\sum_{\vv \in \vU_{i} \setminus \vU_{i,j}}\vu^\top \vv / \tau
    \\&\qquad = \nonumber 
    -\frac{1}{(m-1)n}/\tau
    - \frac{mn-1}{m(m-1)n^2(n-1)p^2}\sumi\sum_{j\in[n]}  \sum_{\vu \in \vU_{i,j}}\sum_{\vv \in \vU_{i} \setminus \vU_{i,j}}\vu^\top \vv / \tau
    \\&\qquad = \nonumber 
    -\frac{1}{(m-1)n}/\tau
    - \frac{mn-1}{m(m-1)n^2(n-1)}\sum_{i\in[m],j \ne j'\in[n]} \E [\vU_{i,j}]^\top\E [\vU_{i,j'}]/ \tau
    \\&\qquad \geq \nonumber 
    -\frac{1}{(m-1)n}/\tau
    - \frac{mn-1}{m(m-1)n^2(n-1)}
    \cdot mn(n-1)/\tau
    \\&\qquad \qquad \label{eq:proof:main:ineq:5}
    + \frac{mn-1}{2m(m-1)n^2(n-1)}
    \sum_{i\in[m],j \ne j'\in[n]}
    \big\|\E[\vU_{i,j}] - \E[\vU_{i,j'}] \big\|_2^2 / \tau
    \\&\qquad = \nonumber 
    -\frac{m}{m-1}/\tau
    - \frac{mn-1}{2m(m-1)n^2(n-1)}
    \sum_{i\in[m],j \ne j'\in[n]}
    \big\|\E[\vU_{i,j}] - \E[\vU_{i,j'}] \big\|_2^2 / \tau,
\end{align}
where the inequality in \eqref{eq:proof:main:ineq:4} comes from \eqref{eq:proof:lemma:inequality:1} in Lemma~\ref{thm:lemma:1} with the equality condition of
\begin{equation}
    \label{eq:proof:main:condition:4}
    \left\|\sum_{\vu \in \vU} \vu \right\|_2^2 = 0,
\end{equation}
and the inequality in \eqref{eq:proof:main:ineq:5} comes from \eqref{eq:proof:lemma:result:2} in Lemma~\ref{thm:lemma:2} with the equality condition of
\begin{equation}
    \label{eq:proof:main:condition:5}
    \big\|\E[\vU_{i,j}]\big\|_2^2 = 1
    \qquad \forall 
    i\in[m],j \in[n]
    .
\end{equation}

Moreover, the value inside the second exponential in \eqref{eq:proof:main:ineq:3} follows

\begin{align}
    &\nonumber
    \frac{1}{mnp^2}\sum_{i\in[m],j \in[n]} \sum_{\vu \in \vU_{i,j}}\sum_{\vw \in \vU_{i,j}}\vu^\top \vw/\tau
    - \frac{1}{mn(n-1)p^2}\sum_{i\in[m],j\in[n]}  \sum_{\vu \in \vU_{i,j}}\sum_{\vv \in \vU_{i} \setminus \vU_{i,j}}\vu^\top \vv / \tau
    \\ 
    &\qquad = \nonumber
    \frac{1}{mn}\sum_{i\in[m],j\in[n]}\norm{\E [\vU_{i,j}]}_2^2 /\tau
    - \frac{1}{mn(n-1)}\sum_{i\in[m],j\ne j'\in[n]}  \E [\vU_{i,j}]^\top\E [\vU_{i,j'}] \tau
    \\ 
    &\qquad = \nonumber
    \frac{1}{2mn}\sum_{i\in[m],j\in[n]}\norm{\E [\vU_{i,j}]}_2^2 /\tau
    +\frac{1}{2mn}\sum_{i\in[m],j'\in[n]}\norm{\E [\vU_{i,j'}]}_2^2 /\tau
    - \frac{1}{2mn(n-1)}\sum_{i\in[m],j\ne j'\in[n]}  2 \E [\vU_{i,j}]^\top\E [\vU_{i,j'}] \tau
    \\&\qquad = \nonumber 
    \frac{1}{2mn(n-1)}\sum_{i\in[m],j \ne j'\in[n]} \big\|\E[\vU_{i,j}] - \E[\vU_{i,j'}] \big\|_2^2  / \tau.
\end{align}

Therefore, applying the above results to \eqref{eq:proof:main:ineq:3} yields
\begin{align}
    \loss_{\op{Sup}}(\vU)
    & \geq \nonumber 
    \log 
    \Bigg(
    -\frac{m}{m-1}/\tau
    - \frac{mn-1}{2m(m-1)n^2(n-1)}
    \sum_{i\in[m],j \ne j'\in[n]}
    \big\|\E[\vU_{i,j}] - \E[\vU_{i,j'}] \big\|_2^2 / \tau
    \bigg)
    \\&\qquad\qquad \label{eq:proof:main:ineq:sup} 
    + (n-1)p
    + p\cdot\exp\bigg(
        \frac{1}{2mn(n-1)}\sum_{i\in[m],j \ne j'\in[n]} \big\|\E[\vU_{i,j}] - \E[\vU_{i,j'}] \big\|_2^2  / \tau
    \bigg)
    \Bigg).
\end{align}

On the other hand, minimizing $\loss_{\op{Self}}(\vU)$ in \eqref{eq:proof:main:loss:assume:cl} in a manner similar to the approach described above yields the following result:

\begin{align}
    \loss_{\op{Self}}(\vU)
    & = \nonumber
    \frac{1}{mnp^2}\sum_{i\in [m], j\in[n]} \sum_{\vu \in \vU_{i,j}}\sum_{\vv \in \vU_{i,j}}
    \log\left(\sum_{\vw \in \vU}\exp(\vu^\top (\vw-\vv) / \tau)\right)
    \\ & \geq \nonumber 
    \frac{1}{mnp^2}
    \sum_{i\in[m], j\in[n]}\sum_{\vu \in \vU_{i,j}}\sum_{\vv \in \vU_{i,j}}
    \log
    \Bigg(
        (m-1)np\cdot
        \exp\bigg(\frac{1}{(m-1)np}\sum_{\vw \in \vU \setminus \vU_i}\vu^\top \vw / \tau - \vu^\top \vv / \tau\bigg)
        \\& \nonumber
        \qquad\qquad\qquad\qquad\qquad\qquad\qquad\qquad\quad
         + (n-1)p\cdot\exp\bigg(\frac{1}{(n-1)p}\sum_{\vw \in \vU_i \setminus \vU_{i,j}}\vu^\top \vw / \tau - \vu^\top\vv / \tau\bigg)
        \\& \label{eq:proof:main:ineq:6}
        \qquad\qquad\qquad\qquad\qquad\qquad\qquad\qquad\quad
         + p\cdot\exp\bigg(\frac{1}{p}\sum_{\vw \in \vU_{i,j}}\vu^\top \vw/\tau - \vu^\top \vv / \tau\bigg)
    \Bigg)
\end{align}
where the inequality in \eqref{eq:proof:main:ineq:6} comes from using Jensen's inequality three times. The equality in \eqref{eq:proof:main:ineq:6} is achieved if there exist some constants $c_4, c_5, c_6 \in \sR$ such that the following conditions hold for all $i \in [m], j\in [n]$, $\vu, \vv\in \vU_{i,j}$:
\begin{align}
    \vu^\top (\vw-\vv)  &= c_{4}
    \qquad\qquad \forall \vw \in \vU \setminus \vU_{i},
    \label{eq:proof:main:condition:6}
    \\
    \vu^\top (\vw-\vv) &= c_{5}
    \qquad\qquad \forall \vw \in \vU_i\setminus \vU_{i,j},
    \label{eq:proof:main:condition:7}
    \\
    \vu^\top (\vw-\vv) &= c_{6}
    \qquad\qquad \forall \vw \in \vU_{i,j}.
    \label{eq:proof:main:condition:8}
\end{align}

From the \eqref{eq:proof:main:ineq:6}, the following holds.
\begin{align}
    \loss_{\op{Self}}(\vU)
    & \geq \nonumber 
    \log
    \Bigg(
        (m-1)np\cdot
        \exp\bigg(
            \frac{1}{m(m-1)n^2p^2}\sum_{i\in[m], j\in[n]}\sum_{\vu \in \vU_{i,j}}\sum_{\vw \in \vU \setminus \vU_i}\vu^\top \vw / \tau 
        \\& \nonumber \qquad\qquad\qquad\qquad\qquad\qquad
            - \frac{1}{mnp^2}\sum_{i\in[m], j\in[n]}\sum_{\vu \in \vU_{i,j}}\sum_{\vv \in \vU_{i,j}}\vu^\top \vv / \tau\bigg)
        \\& \nonumber
        \qquad\qquad + (n-1)p\cdot
        \exp\bigg(
            \frac{1}{mn(n-1)p^2}\sum_{i\in[m], j\in[n]}\sum_{\vu \in \vU_{i,j}}\sum_{\vw \in \vU_i \setminus \vU_{i,j}}\vu^\top \vw / \tau
            \\& \label{eq:proof:main:ineq:7} \qquad\qquad\qquad\qquad\qquad\qquad
            - \frac{1}{mnp^2}\sum_{i\in[m], j\in[n]}\sum_{\vu \in \vU_{i,j}}\sum_{\vv \in \vU_{i,j}}\vu^\top\vv / \tau\bigg)
         + p
    \Bigg)
    \\
     & = \nonumber 
    \log
    \Bigg(
        (m-1)np\cdot
        \exp\bigg(
            \frac{1}{m(m-1)n^2p^2}\sum_{i\in[m], j\in[n]}\sum_{\vu \in \vU_{i,j}}\sum_{\vw \in \vU \setminus \vU_i}\vu^\top \vw / \tau 
        \\& \nonumber \qquad\qquad\qquad\qquad\qquad\qquad
            - \frac{1}{mnp^2}\sum_{i\in[m], j\in[n]}\sum_{\vu \in \vU_{i,j}}\sum_{\vv \in \vU_{i,j}}\vu^\top \vv / \tau\bigg)
        \\& \nonumber
        \qquad\quad + (n-1)p\cdot
        \exp\bigg(
            - \frac{1}{2mn(n-1)} \sum_{i\in[m],j \ne j'\in[n]} \big\|\E[\vU_{i,j}] - \E[\vU_{i,j'}] \big\|_2^2 / \tau\bigg)
         + p
    \Bigg)
    \\ & \geq \nonumber 
    \log
    \Bigg(
        (m-1)np \cdot
        \exp\bigg(
            \frac{1}{m(m-1)n^2p^2}\sum_{i\in[m], j\in[n]}\sum_{\vu \in \vU_{i,j}}\sum_{\vw \in \vU \setminus \vU_i}\vu^\top \vw / \tau 
            - 1 /\tau \bigg)
        \\& \label{eq:proof:main:ineq:8}
        \qquad\quad + (n-1)p \cdot
        \exp\bigg(
            - \frac{1}{2mn(n-1)} \sum_{i\in[m],j \ne j'\in[n]} \big\|\E[\vU_{i,j}] - \E[\vU_{i,j'}] \big\|_2^2 / \tau\bigg)
         + p
    \Bigg)
    \end{align}
    where the inequality in \eqref{eq:proof:main:ineq:7} holds from Proposition~\ref{thm:multivariate:jensen}, and the equality holds if \eqref{eq:proof:main:condition:7}, \eqref{eq:proof:main:condition:7}, and \eqref{eq:proof:main:condition:8} are satisfied. Moreover, the equality condition for \eqref{eq:proof:main:ineq:8} is equivalent to \eqref{eq:proof:main:condition:5}.

By using \eqref{eq:proof:lemma:inequality:1} in Lemma~\ref{thm:lemma:1} and \eqref{eq:proof:lemma:result:2} in Lemma~\ref{thm:lemma:2},
\begin{align}
\loss_{\op{Self}} (\vU) & \geq \nonumber 
    \log
    \Bigg(
        (m-1)np\cdot
        \exp\bigg(
            - \frac{(m-1)n+1}{(m-1)n} /\tau
            -\frac{1}{m(m-1)n^2p^2}\sum_{i\in[m],j \ne j'\in[n]} \sum_{\vu \in \vU_{i,j}}\sum_{\vw \in \vU_{i,j'}}\vu^\top \vw / \tau\bigg)
        \\& \label{eq:proof:main:ineq:9}
        \qquad\quad + (n-1)p\cdot
        \exp\bigg(
            - \frac{1}{2mn(n-1)} \sum_{i\in[m],j \ne j'\in[n]} \big\|\E[\vU_{i,j}] - \E[\vU_{i,j'}] \big\|_2^2 / \tau\bigg)
         + p
    \Bigg)
    \\ & \geq \nonumber 
    \log
    \Bigg(
        (m-1)np\cdot
        \exp\bigg(
            - \frac{m}{m-1} /\tau
            +\frac{1}{2m(m-1)n^2} 
            \sum_{i\in[m],j \ne j'\in[n]} \big\|\E[\vU_{i,j}] - \E[\vU_{i,j'}] \big\|_2^2 / \tau \bigg)
        \\& \label{eq:proof:main:ineq:10}
        \qquad\quad + (n-1)p\cdot
        \exp\bigg(
            - \frac{1}{2mn(n-1)} \sum_{i\in[m],j \ne j'\in[n]} \big\|\E[\vU_{i,j}] - \E[\vU_{i,j'}] \big\|_2^2 / \tau\bigg)
         + p
    \Bigg)
    ,
\end{align}
where the equality conditions of  \eqref{eq:proof:main:ineq:9} and \eqref{eq:proof:main:ineq:10} are fulfilled if \eqref{eq:proof:main:condition:4} and \eqref{eq:proof:main:condition:5} are satisfied.

Finally, by combining the results of minimizing each loss in \eqref{eq:proof:main:ineq:sup} and \eqref{eq:proof:main:ineq:10}, we obtain:
\begin{align}
    \loss (\vU)
    &= \nonumber
    (1-\alpha)\; \loss_{\op{Sup}}(\vU)+\alpha\; \loss_{\op{Self}}(\vU)
    \\ & \geq \nonumber 
    (1-\alpha)
    \log 
    \Bigg(
    -\frac{m}{m-1}/\tau
    - \frac{mn-1}{2m(m-1)n^2(n-1)}
    \sum_{i\in[m],j \ne j'\in[n]}
    \big\|\E[\vU_{i,j}] - \E[\vU_{i,j'}] \big\|_2^2 / \tau
    \bigg)
    \\&\qquad\qquad\qquad\qquad \nonumber
    + (n-1)p
    + p\cdot\exp\bigg(
        \frac{1}{2mn(n-1)}\sum_{i\in[m],j \ne j'\in[n]} \big\|\E[\vU_{i,j}] - \E[\vU_{i,j'}] \big\|_2^2  / \tau
    \bigg)
    \Bigg)
    \\ & \quad \nonumber +
    \alpha
    \log
    \Bigg(
        (m-1)np\cdot
        \exp\bigg(
            - \frac{m}{m-1} /\tau
            +\frac{1}{2m(m-1)n^2} 
            \sum_{i\in[m],j \ne j'\in[n]} \big\|\E[\vU_{i,j}] - \E[\vU_{i,j'}] \big\|_2^2 / \tau \bigg)
        \\& \label{eq:main:proof:loss:result}
        \qquad\qquad\qquad\ + (n-1)p\cdot
        \exp\bigg(
            - \frac{1}{2mn(n-1)} \sum_{i\in[m],j \ne j'\in[n]} \big\|\E[\vU_{i,j}] - \E[\vU_{i,j'}] \big\|_2^2 / \tau\bigg)
         + p
    \Bigg)
    \\&:= \label{eq:main:proof:loss:result:func}
        l\left( \sum_{i\in[m],j \ne j'\in[n]} \big\|\E[\vU_{i,j}] - \E[\vU_{i,j'}] \big\|_2^2 \right),
\end{align}
where the function $l$ is defined for simple notation, as \eqref{eq:main:proof:loss:result} depends on the term $\sum_{i\in[m],j \ne j'\in[n]} \big\|\E[\vU_{i,j}] - \E[\vU_{i,j'}] \big\|_2^2$.

Note that the \ours{} $\vU^\delta$ in Def.~\ref{def:model:general} satisfies all of equality conditions in \eqref{eq:proof:main:condition:1}, \eqref{eq:proof:main:condition:2}, \eqref{eq:proof:main:condition:3}, \eqref{eq:proof:main:condition:4}, \eqref{eq:proof:main:condition:5}, \eqref{eq:proof:main:condition:6}, \eqref{eq:proof:main:condition:7}, and \eqref{eq:proof:main:condition:8}. Moreover, the \ours{} $\vU^\delta$ in Def.~\ref{def:model:general} can attain all possible values of $\sum_{i\in[m],j \ne j'\in[n]} \big\|\E[\vU_{i,j}] - \E[\vU_{i,j'}] \big\|_2^2$.
As a result, the equality in \eqref{eq:main:proof:loss:result} holds when $\vU$ is substituted by $\vU^\delta$. That is,
\begin{align*}
    \loss(\vU)
    &\geq
    l\left( \sum_{i\in[m],j \ne j'\in[n]} \big\|\E[\vU_{i,j}] - \E[\vU_{i,j'}] \big\|_2^2 \right)
    \\&\geq
    \min_{\vU} \; l\left( \sum_{i\in[m],j \ne j'\in[n]} \big\|\E[\vU_{i,j}] - \E[\vU_{i,j'}] \big\|_2^2 \right)
    \\&=
    \min_{\vU^\delta} \; l\left( \sum_{i\in[m],j \ne j'\in[n]} \big\|\E[\vU_{i,j}] - \E[\vU_{i,j'}] \big\|_2^2 \right)
    \\&=\min_{\vU^\delta}\loss(\vU^\delta)
    \\&=\min_{\delta}\loss(\vU^\delta).
\end{align*}

Now, we only need to determine $\delta^\star$ of \ours{} in Def.~\ref{def:model:general} such that minimize $\loss (\vU^{\delta^\star})$ in \eqref{eq:loss}. Note that the denominator in each logarithm in $\loss_{\op{Sup}}(\vU^\delta)$ and $\loss_{\op{Self}}(\vU^\delta)$ has the same value for a given  $\delta$, as follows. For any $\vu^\delta \in \vU^\delta$, by using \eqref{eq:model:same:general}-\eqref{eq:model:different:general} in Def.~\ref{def:model:general},

\begin{align}
\sum_{\vw^\delta \in \vU^\delta}\exp\left(\left(\vu^\delta\right)^\top \vw^\delta / \tau\right)
& = \nonumber
p\cdot\exp(1/\tau) + (n-1)p \cdot\exp\left( \left( 1 - \delta^2 \frac{mn}{mn-1}  \right) /\tau\right) 
\\ & \quad \nonumber
+ (m-1)np\cdot \exp\left(\left(- \frac{1}{m-1} + \delta^2 \frac{m(n-1)}{(m-1)(mn-1)}\right) / \tau\right)
\\ & := g(\delta;m,n,p,\tau), \label{eq:main:g:denominator}
\end{align}
where $g(\delta;m,n,p,\tau)$ in \eqref{eq:main:g:denominator} is defined for the simple notation.

Then, the losses $\loss_{\op{Sup}}(\vU^\delta)$ and $\loss_{\op{Self}}(\vU^\delta)$ are simplified as follows.
\begin{align*}
\loss_{\op{Sup}}(\vU^{\delta})
& = 
- \frac{1}{mn(n-1)p^2} \sum_{i\in[m], j\ne j'\in[n]}\sum_{\vu \in \vU_{i,j}}\sum_{\vv \in \vU_{i,j'}}
\log \frac{\exp(\vu^\top \vv / \tau)}{\sum_{\vw \in \vU}\exp(\vu^\top \vw / \tau)}
\\ & = 
- \frac{1}{mn(n-1)p^2} \sum_{i\in[m], j\ne j'\in[n]}\sum_{\vu \in \vU_{i,j}}\sum_{\vv \in \vU_{i,j'}}
\log \frac{\exp\big(\big(1-\delta^2\frac{mn}{mn-1}\big)/ \tau\big)}{g(\delta;m,n,p,\tau)}
\\ & = 
\log \frac{g(\delta;m,n,p,\tau)}{\exp\big(\big(1-\delta^2\frac{mn}{mn-1}\big)/ \tau\big)},
\\  
\loss_{\op{Self}}(\vU^{\delta})
& = 
-\frac{1}{mnp^2} \sum_{i\in[m], j\in[n]}\sum_{\vu \in \vU_{i,j}}\sum_{\vv \in \vU_{i,j}}
\log \frac{\exp(\vu^\top \vv / \tau)}{ \sum_{\vw \in \vU}\exp(\vu^\top \vw / \tau)}
\\ & = 
-\frac{1}{mnp^2} \sum_{i\in[m], j\in[n]}\sum_{\vu \in \vU_{i,j}}\sum_{\vv \in \vU_{i,j}}
\log \frac{\exp(1 / \tau)}{g(\delta;m,n,p,\tau)}
\\ & = \log \frac{g(\delta;m,n,p,\tau)}{\exp(1 / \tau)}.
\end{align*}

As a result, $\loss(\vU^\delta)$ is rewritten as
\begin{align}
    \loss (\vU^{\delta})
    &= \nonumber
    (1-\alpha)\; \loss_{\op{Sup}}(\vU^{\delta})+\alpha\; \loss_{\op{Self}}(\vU^{\delta})
    \\
    &= \nonumber
    (1-\alpha)\cdot
    \log \frac{g(\delta;m,n,p,\tau)}{\exp\left(\left(1-\delta^2\frac{mn}{mn-1}\right)/ \tau\right)}
    +
    \alpha\cdot
    \log \frac{g(\delta;m,n,p,\tau)}{\exp(1/ \tau)}
    \\ 
    &= \nonumber
    \log g(\delta;m,n,p,\tau)
    -(1-\alpha) \left(1-\delta^2\frac{mn}{mn-1}\right)/\tau
    - \alpha/ \tau
    \\ 
    &= \nonumber
    \log \bigg(p\cdot\exp(1/\tau) + (n-1)p\cdot\exp\left( \left( 1 - \delta^2 \frac{mn}{mn-1}  \right) /\tau\right) 
    \\ & \nonumber \qquad \quad
    + (m-1)np\cdot \exp\left(\left(- \frac{1}{m-1} + \delta^2 \frac{m(n-1)}{(m-1)(mn-1)}\right) / \tau\right) \bigg) 
    \\& \nonumber\quad
    -(1-\alpha) \cdot\left(1-\delta^2\frac{mn}{mn-1}\right)/\tau
    - \alpha \cdot(1/ \tau)
    \\ \nonumber 
    &=
    \log \bigg(1 + (n-1)\cdot\exp\left( -\Tilde{\delta} /\tau\right) 
    + (m-1)n\cdot \exp\left(\left(- \frac{m}{m-1} + \Tilde{\delta} \frac{n-1}{(m-1)n}\right) / \tau\right) \bigg) 
    \\& \label{eqn:proof:delta:tilde} \quad
    + \log \left( p \cdot\exp(1/\tau)\right)
    -(1-\alpha) (1-\Tilde{\delta})/\tau
    - \alpha / \tau,
    \\ \nonumber 
    &=
    \log \bigg(1 + (n-1)\cdot\exp\left( -\Tilde{\delta} /\tau\right) 
    + (m-1)n\cdot \exp\left(\left(- \frac{m}{m-1} + \Tilde{\delta} \frac{n-1}{(m-1)n}\right) / \tau\right) \bigg) 
    + \log  p
    +(1-\alpha) \Tilde{\delta}/\tau,
\end{align}
where $\Tilde{\delta} := \delta^2\frac{mn}{mn-1} \in \big[ 0, \frac{n}{n-1} \big]$ in \eqref{eqn:proof:delta:tilde} is the monotonic increasing transformation of $\delta\in \left[ 0, \sqrt{\frac{mn-1}{m(n-1)}} \right]$.

To find $\Tilde{\delta}^\star$ that minimize $\loss \big(\vU^{\Tilde{\delta}}\big)$,
\begin{align}
\frac{\partial}{\partial\Tilde{\delta}} \loss \big(\vU^{\Tilde{\delta}}\big)
    &= \nonumber
    \frac{
        -\frac{n-1}{\tau}\cdot\exp\left( -\Tilde{\delta} /\tau\right) 
        +\frac{n-1}{\tau} \cdot\exp\left(\left(- \frac{m}{m-1} + \Tilde{\delta} \frac{n-1}{(m-1)n}\right) / \tau\right)
    }{
        1 + (n-1)\cdot\exp\left( -\Tilde{\delta} /\tau\right) 
        + (m-1)n \cdot\exp\left(\left(- \frac{m}{m-1} + \Tilde{\delta} \frac{n-1}{(m-1)n}\right) / \tau\right)
    }
    +(1-\alpha) /\tau
    \\
    &= \nonumber
    \frac{1}{\tau} \cdot
    \frac{
        -(n-1)\cdot\exp\left( -\Tilde{\delta} /\tau\right) 
        +(n-1) \cdot\exp\left(\left(- \frac{m}{m-1} + \Tilde{\delta} \frac{n-1}{(m-1)n}\right) / \tau\right)
    }{
        1 + (n-1)\cdot\exp\left( -\Tilde{\delta} /\tau\right) 
        + (m-1)n \cdot\exp\left(\left(- \frac{m}{m-1} + \Tilde{\delta} \frac{n-1}{(m-1)n}\right) / \tau\right)
    }
    +\frac{1}{\tau} \cdot (1-\alpha) 
    \\  \label{eq:proof:main:last:tilde:delta} 
    &=
    \frac{1}{\tau} \cdot
    \frac{
    h\big(\Tilde{\delta}; m,n,\tau,\alpha\big)
    }{
    1 + (n-1) \cdot\exp( - \Tilde{\delta} /\tau) 
    + (m-1)n \cdot\exp\left(\left(- \frac{m}{m-1} + \Tilde{\delta} \frac{n-1}{(m-1)n}\right) / \tau\right)
    },
\end{align}
where we define the function $h$ in \eqref{eq:proof:main:last:tilde:delta} as
\begin{align*}
h\big(\Tilde{\delta}; m,n,\tau,\alpha\big)
&=
(1-\alpha)
- \alpha(n-1)\cdot\exp( - \Tilde{\delta} /\tau)
\\
&\quad
+ (mn-1-\alpha (m-1)n)\cdot \exp\left(\left(- \frac{m}{m-1} + \Tilde{\delta} \frac{n-1}{(m-1)n}\right) / \tau\right).    
\end{align*}

Note that the denominator in \eqref{eq:proof:main:last:tilde:delta} is always positive. Moreover, the function $h\big(\Tilde{\delta}; m,n,\tau,\alpha\big)$ in \eqref{eq:proof:main:last:tilde:delta} is monotonically increasing with respect to $\Tilde{\delta} \in \big[ 0, \frac{n}{n-1} \big]$, and the following value is non-negative.
\begin{align*}
h\left(\frac{mn}{mn-1}; m,n,\tau,\alpha\right)
&=
(1-\alpha)
- \alpha(n-1)\cdot\exp\left( - \frac{mn}{mn-1}/\tau\right)
\\
&\quad
+ (mn-1-\alpha (m-1)n)\cdot \exp\left(\left(- \frac{m}{m-1} + \frac{m(n-1)}{(m-1)(mn-1)}\right) / \tau\right)
\\ &=
(1-\alpha)
- \alpha(n-1)\cdot\exp\left(- \frac{mn}{mn-1}/\tau\right)
+ (mn-1-\alpha (m-1)n) \cdot\exp\left(- \frac{mn}{mn-1}/ \tau\right)
\\ &=
(1-\alpha)
+ (mn-1)(1-\alpha) \cdot\exp\left(- \frac{mn}{mn-1}/ \tau\right)
\\ &\geq 0.
\end{align*}

Therefore, $\Tilde{\delta}^\star$, which minimizes $\loss \big(\vU^{\Tilde{\delta}}\big)$, can be determined as follows:
\begin{equation}
    \nonumber
    \Tilde{\delta}^\star = 
    \begin{cases}
        0,
        &\text{if } h\big(0; m,n,\tau,\alpha\big)\geq0,
        \\
        \Tilde{\delta} \in \left( 0, \frac{mn}{mn-1}\right)
        \text{ such that }h\big(\Tilde{\delta}; m,n,\tau,\alpha\big)=0,
        &\text{otherwise.}
    \end{cases}
\end{equation}

This can be rewritten as:
\begin{equation}
    \nonumber
    \delta^\star = 
    \begin{cases}
        0,
        &\text{if } h\big(0; m,n,\tau,\alpha\big)\geq0,
        \\
        \delta\in \left( 0, 1\right],
        \text{ such that }h\left(\delta^2\frac{mn}{mn-1}; m,n,\tau,\alpha\right)=0,
        &\text{otherwise}.
    \end{cases}
\end{equation}

As a result, all embedding sets $\vU^{\star}$ that minimize the loss $\loss(\vU)$ in \eqref{eq:loss} are included in the \ours{} as follows:
\begin{align}
    \nonumber
    \forall \vU^{\star} \in \argmin_{\vU} \loss(\vU), \quad \exists! \delta \in [0,1] \text{ such that } \vU^{\delta} = \vU^{\star}.
\end{align}
 where the uniqueness of $\delta$ arises from the monotonicity of $h$.

\end{proof}

\newpage
\subsection{Preventing Class Collapse}
\label{appendix:proof:class:collapse}

\begin{theorem}
    Let $\vU^{\star}$ be the set of optimal embedding vectors that minimizes the loss $\loss(\vU)$ in \eqref{eq:loss}.
    Then, the class collapse does not happen, \ie $\Var[\vU_i^{\star}]>0$ for all $i\in[m]$, if and only if 
    the loss-combining coefficient $\alpha$ satisfies
    \begin{equation}
        \nonumber
        \alpha
        \in \left(
        \frac{mn-1 + \exp\big( \frac{m}{m-1} / \tau\big)}{mn-n + n\cdot \exp\big( \frac{m}{m-1} / \tau\big)}
        , 1 \right]
    \end{equation}
    for a given temperature $\tau >0$.
    This necessary and sufficient condition for preventing class-collapse can be re-written as 
    \begin{equation*}
        \tau
        \in
        \left(0, 
        \frac{1}{(1-\frac{1}{m}) \cdot
        \log\big( \frac{mn-1 -\alpha (m-1)n}{ \alpha n - 1}
        \big)}
        \right)
    \end{equation*}
    for a given $\alpha \in \big(\frac{1}{n},1\big]$.
\end{theorem}

\begin{proof}
To find the condition for preventing class-collapse. Proposition~\ref{thm:app:ssem:variance} implies that $\delta^\star$ of \ours{} must be positive. Therefore,  from \eqref{eq:proof:find:optimal:delta} in Theorem~\ref{thm:app:optimal:ssem}, the necessity and sufficient condition of preventing class-collapse is $h\big(0; m,n,\tau\big)<0$ where
\begin{align*}
h\big(x; m,n,\tau,\alpha\big)
&=
(1-\alpha)
- \alpha(n-1)\cdot\exp( - x /\tau)
\\
&\quad
+ (mn-1-\alpha (m-1)n) \cdot\exp\left(\left(- \frac{m}{m-1} + x \frac{n-1}{(m-1)n}\right) / \tau\right).  
\end{align*}
This condition can be rewritten as follows:
\begin{align}
    0
    & \nonumber > 
    h\big(0; m,n,\tau,\alpha\big)
    \\ \nonumber &=    
    (1-\alpha)
    -  \alpha(n-1) \cdot\exp(0)
    + (mn-1 -\alpha (m-1)n) \cdot\exp\bigg(- \frac{m}{m-1} / \tau\bigg)
    \\ \label{eq:proof:from:here:1} &=    
    1-\alpha n
    + (mn-1 -\alpha (m-1)n) \cdot\exp\bigg(- \frac{m}{m-1} / \tau\bigg)
    \\ \nonumber &=    
    -\alpha n \bigg( 1+ (m-1)\cdot \exp\bigg(- \frac{m}{m-1} / \tau\bigg)\bigg)
    +1 + (mn-1)\cdot\exp\bigg(- \frac{m}{m-1} / \tau\bigg),
\end{align}
which is equal to
\begin{align*}
    \alpha
    &>
    \frac{1 + (mn-1)\cdot\exp\big(- \frac{m}{m-1} / \tau\big)}{ n \big( 1+ (m-1) \cdot\exp\big(- \frac{m}{m-1} / \tau\big)\big)}
    =
    \frac{mn-1 + \exp\big( \frac{m}{m-1} / \tau\big)}{mn-n + n\cdot\exp\big( \frac{m}{m-1} / \tau\big)}
    .
\end{align*}

Or equivalently, from \eqref{eq:proof:from:here:1},
\begin{align}
    \label{eq:proof:from:here:2}
    \alpha n - 1
    & > 
    (mn-1 -\alpha (m-1)n) 
     \cdot\exp\bigg(- \frac{m}{m-1} / \tau\bigg).
\end{align}

Note that
\begin{align*}
    \frac{mn-1}{mn-n} \geq \frac{mn-1}{mn-1} =1 \geq \alpha,
\end{align*}
which implies $mn-1 -\alpha (m-1)n\geq 0$. Moreover,
\begin{align*}
    \frac{
        \alpha n - 1
    }{
        mn-1 -\alpha (m-1)n
    }
    =
    \frac{
        \alpha n - 1
    }{
        \alpha n -1 + mn(1-\alpha)
    }
    \leq
    \frac{
        \alpha n - 1
    }{
        \alpha n -1 
    }
    =1.
\end{align*}

Then the following conditions are equivalent to \eqref{eq:proof:from:here:2}:
\begin{align*}
    \frac{
        \alpha n - 1
    }{
        mn-1 -\alpha (m-1)n
    }
    & >
     \exp\bigg(- \frac{m}{m-1} / \tau\bigg),
\end{align*}
\begin{align*}
    \log\bigg(
    \frac{
        \alpha n - 1
    }{
        mn-1 -\alpha (m-1)n
    }
    \bigg)
    & >
    - \frac{m}{m-1} / \tau,
\end{align*}
\begin{align*}
    \tau
    & <
    \frac{m}{
    (m-1)\cdot
    \log\big(
    \frac{
        mn-1 -\alpha (m-1)n
    }{
        \alpha n - 1
    }
    \big)
    }
    =
    \left(0, 
    \frac{1}{(1-\frac{1}{m})\cdot
    \log\big( \frac{mn-1 -\alpha (m-1)n}{ \alpha n - 1}
    \big)}
    \right).
\end{align*}

Note that the minimum range of $\alpha$ is $\frac{1}{n}$, which is comes from
\begin{equation*}
    \frac{1}{n} = \lim_{\tau \to 0} 
    \frac{mn-1 + \exp\big( \frac{m}{m-1} / \tau\big)}{mn-n + n\cdot\exp\big( \frac{m}{m-1} / \tau\big)}.
\end{equation*}

\end{proof}

\vspace{20pt}

\subsection{The Optimality of Class-Conditional InfoNCE Loss}
\label{appendix:proof:optimality:other}

\begin{proposition}
Let the class-conditional InfoNCE loss $\loss_{\op{cNCE}}(\vU)$ be defined as
\begin{equation}
    \nonumber
    \loss_{\op{cNCE}}(\vU)
     = -\frac{1}{mnp^2} \sum_{i\in[m], j\in[n]}\sum_{\vu \in \vU_{i,j}}\sum_{\vv \in \vU_{i,j}}
     \log \frac{\exp(\vu^\top \vv / \tau)}{ \sum_{\vw \in \vU_i}\exp(\vu^\top \vw / \tau)}.
\end{equation}
Suppose the embedding dimension satisfies $d \geq mn - 1$ for given $m, n \in \mathbb{N}$ with $mn \geq 2$. Then, for any embedding set $\vU$,
\begin{equation}
    \nonumber
    \loss_{\op{cNCE}}(\vU) \geq \loss_{\op{cNCE}}(\vU^\delta)
\end{equation}
holds for $\delta = \sqrt{\frac{mn - 1}{m(n - 1)}}$. \end{proposition}
\begin{proof}

We minimize $\loss_{\op{cNCE}}(\vU)$ using a similar approach as in the proofs of Theorem~\ref{thm:app:optimal:ssem}.

\begin{align}
     \loss_{\op{cNCE}}(\vU)
     &= \nonumber
     \frac{1}{mnp^2} \sum_{i\in[m], j\in[n]}\sum_{\vu \in \vU_{i,j}}\sum_{\vv \in \vU_{i,j}}
     \log
     \left(\sum_{\vw \in \vU_i}\exp(\vu^\top (\vw-\vv) / \tau)\right)
    \\ & \geq \nonumber 
    \frac{1}{mnp^2}
    \sum_{i\in[m], j\in[n]}\sum_{\vu \in \vU_{i,j}}\sum_{\vv \in \vU_{i,j}}
    \log
    \Bigg(
         (n-1)p\cdot\exp\bigg(\frac{1}{(n-1)p}\sum_{\vw \in \vU_i \setminus \vU_{i,j}}\vu^\top \vw / \tau - \vu^\top\vv / \tau\bigg)
        \\& \label{eq:proof:cnce:ineq:1}
        \qquad\qquad\qquad\qquad\qquad\qquad\qquad\qquad
         + p\cdot \exp\bigg(\frac{1}{p}\sum_{\vw \in \vU_{i,j}}\vu^\top \vw/\tau - \vu^\top \vv / \tau\bigg)
    \Bigg)
\end{align}
where the inequality in \eqref{eq:proof:cnce:ineq:1} comes from using Jensen's inequality two times. The equality in \eqref{eq:proof:cnce:ineq:1} is achieved if there exist some constants $c_1, c_2 \in \sR$ such that the following conditions hold for all $i \in [m], j\in [n]$, $\vu, \vv\in \vU_{i,j}$:
\begin{align}
    \vu^\top (\vw-\vv) &= c_{1}
    \qquad\qquad \forall \vw \in \vU_i\setminus \vU_{i,j},
    \label{eq:proof:cnce:condition:1}
    \\
    \vu^\top (\vw-\vv) &= c_{2}
    \qquad\qquad \forall \vw \in \vU_{i,j}.
    \label{eq:proof:cnce:condition:2}
\end{align}

Then, the following holds,
\begin{align}
    \loss_{\op{cNCE}}(\vU)
    & \geq \nonumber 
    \log
    \Bigg(
        (n-1) p
        \exp\bigg(
            \frac{1}{mn(n-1)p^2}\sum_{i\in[m], j\in[n]}\sum_{\vu \in \vU_{i,j}}\sum_{\vw \in \vU_i \setminus \vU_{i,j}}\vu^\top \vw / \tau
            \\& \label{eq:proof:cnce:ineq:2} \qquad\qquad\qquad\qquad\qquad
            - \frac{1}{mnp^2}\sum_{i\in[m], j\in[n]}\sum_{\vu \in \vU_{i,j}}\sum_{\vv \in \vU_{i,j}}\vu^\top\vv / \tau\bigg)
            +p
    \Bigg)
    \\
     & = \nonumber 
    \log
    \Bigg(
        (n-1)p
        \exp\bigg(
            - \frac{1}{2mn(n-1)} \sum_{i\in[m],j \ne j'\in[n]} \big\|\E[\vU_{i,j}] - \E[\vU_{i,j'}] \big\|_2^2 / \tau\bigg)
         + p
    \Bigg)
    \end{align}
    where the inequality in \eqref{eq:proof:cnce:ineq:2} holds from Proposition~\ref{thm:multivariate:jensen}, and the equality holds if \eqref{eq:proof:cnce:condition:1} and \eqref{eq:proof:cnce:condition:2} are satisfied.

Note that the \ours{} $\vU^\delta$ in Def.~\ref{def:model:general} satisfies all of equality conditions in \eqref{eq:proof:cnce:condition:1} and \eqref{eq:proof:cnce:condition:2}. Moreover, the \ours{} $\vU^\delta$ in Def.~\ref{def:model:general} can attain all possible values of $\sum_{i\in[m],j \ne j'\in[n]} \big\|\E[\vU_{i,j}] - \E[\vU_{i,j'}] \big\|_2^2$.
As a result, we only need to find $\delta \in \left[0, \sqrt{\frac{mn-1}{m(n-1)}}\right]$ that minimizes the loss, as below.
\begin{align*}
    \loss_{\op{cNCE}}(\vU)
    &\geq
    \min_{\delta}\loss_{\op{cNCE}}(\vU^\delta).
\end{align*}

Since the \ours{} $\vU^\delta$ satisfy the equality conditions of \eqref{eq:proof:cnce:ineq:1} and \eqref{eq:proof:cnce:ineq:2}, $\loss_{\op{cNCE}}(\vU^\delta)$ is rewritten as follows:
\begin{align*}
    \loss_{\op{cNCE}} (\vU^{\delta})
    &=
     \frac{1}{mnp^2} \sum_{i\in[m], j\in[n]}\sum_{\vu^{\delta} \in \vU_{i,j}^{\delta}}\sum_{\vv^{\delta} \in \vU_{i,j}^{\delta}}
     \log
     \left(\sum_{\vw^{\delta} \in \vU_i^{\delta}}\exp\left(\left(\vu^{\delta}\right)^\top \left(\vw^{\delta}-\vv^{\delta}\right) / \tau\right)\right)
    \\ 
    &= %
    \log
    \Bigg(
        (n-1) p\cdot
        \exp\bigg(
            \frac{1}{mn(n-1)p^2}\sum_{i\in[m], j\in[n]}\sum_{\vu^{\delta} \in \vU_{i,j}^{\delta}}\sum_{\vw^{\delta} \in \vU_i^{\delta} \setminus \vU_{i,j}^{\delta}}(\vu^{\delta})^\top \vw^{\delta} / \tau
            \\&
            \qquad\qquad\qquad\qquad\qquad\quad
            - \frac{1}{mnp^2}\sum_{i\in[m], j\in[n]}\sum_{\vu^{\delta} \in \vU_{i,j}^{\delta}}\sum_{\vv^{\delta} \in \vU_{i,j}^{\delta}}(\vu^{\delta})^\top\vv^{\delta} / \tau\bigg)
            +p
    \Bigg)
    \\ 
    &= %
    \log
    \Bigg(
        (n-1) p\cdot
        \exp\bigg(
            \left(1 - \delta^2 \frac{mn}{mn-1}\right)
            / \tau
            -1 / \tau\bigg)
            +p
    \Bigg)
    \\ 
    &= %
    \log
    \Bigg(
        (n-1) p\cdot
        \exp\bigg(
            - \delta^2 \frac{mn}{mn-1} / \tau
        \bigg)
        +p
    \Bigg)
    ,
\end{align*}
which is a monotonic decreasing function of $\delta$.
Since $\delta$ lies within the range of $\left[ 0, \sqrt{\frac{mn-1}{m(n-1)}} \right]$, the minimum loss is attained when $\delta=\frac{mn-1}{m(n-1)}$.
Therefore, for any embedding set $\vU$,
\begin{equation}
    \nonumber
    \loss_{\op{cNCE}}(\vU) \geq \loss_{\op{cNCE}}(\vU^\delta)
\end{equation}
holds for $\delta = \sqrt{\frac{mn - 1}{m(n - 1)}}$. 

\end{proof}

\newpage

\section{ADDITIONAL EXPERIMENTS ON SYNTHETIC DATA}
\label{sec:appendix:additional:synthetic}

In this section, we present additional experiments on synthetic data. We follow the training setup in Sec.~\ref{sec:experiment:synthetic}, except for using the lower embedding dimension where $d=50$.

Note that Theorem~\ref{thm:optimal:embedding} assumes $d\geq mn-1$, while these additional experiments do not strictly satisfy this condition, as $50=d < mn-1 = 99$. This suggests that even when the optimal embedding set of \ours{} cannot exist in a lower-dimensional embedding space by Proposition~\ref{thm:exist:model}, the average within-class variance of the learned embedding set still aligns with our theoretical analysis as shown in Fig.~\ref{fig:syn:additional}.

\begin{figure}[h]
    \centering
    \includegraphics[width=0.5\textwidth]{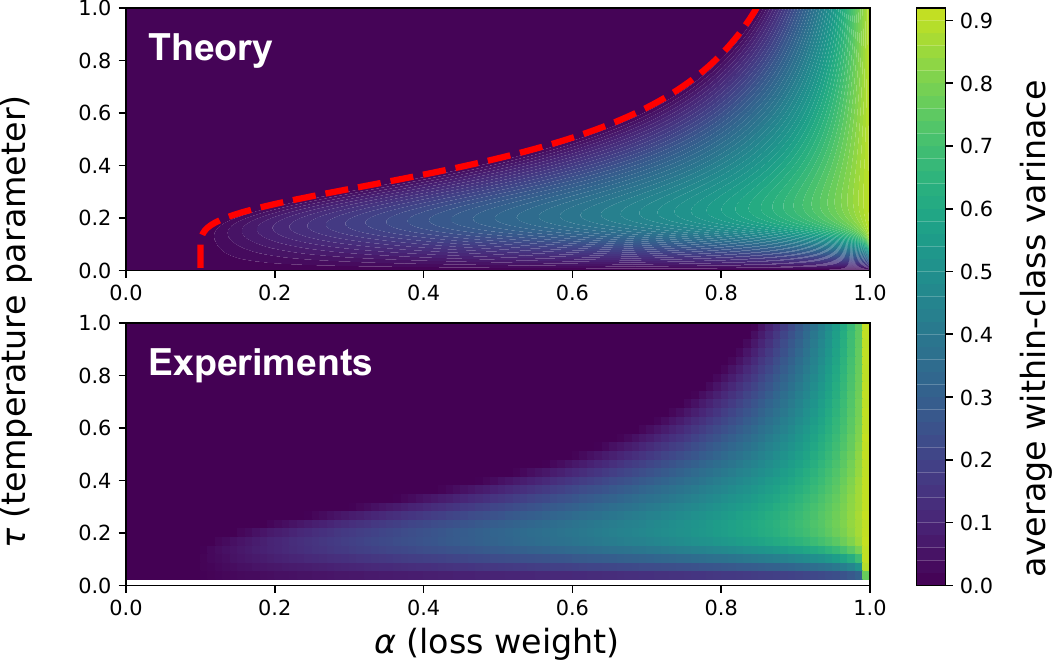}
    \caption{
    The within-class variance (averaged over different classes) of the learned embedding set $\vU$, for various loss-combining coefficient $\alpha$ and temperature $\tau$. (Top): Computed from theoretical results in Sec.~\ref{sec:emb:var}, (Bottom): Computed from the experiments on synthetic datasets in Appendix~\ref{sec:appendix:additional:synthetic}.
    }
    \label{fig:syn:additional}
\end{figure}

\section{EXPERIMENTS ON REAL DATA}
\label{appendix:real:exp}

\subsection{Details of Datasets and Augmentation}
\label{appendix:real:data}

We use the CIFAR-10 and ImageNet-100 datasets \citep{krizhevsky2009cifar, deng2009imagenet} for training the model. For CIFAR-10, we use the balanced mini-batch where the number of instances per class is equal. For ImageNet-100, we select a subset of 100 classes.

For data augmentation strategy, we follow prior works \citep{chen2020simple, he2020momentum}, including random cropping, color jitter, random grayscale, Gaussian blur, and random horizontal flipping.

To evaluate the performance of transfer learning, we use various downstream datasets, as described in Table~\ref{table:data}.

\begin{table}[h]
\caption{Real datasets used in experiments.}
\label{table:data}
\begin{center}
\resizebox{\columnwidth}{!}{%
\begin{tabular}{c|rrrrc}
\toprule
Name & \multicolumn{1}{c}{\# of classes} & \multicolumn{1}{c}{Training size} & \multicolumn{1}{c}{Validation size} & \multicolumn{1}{c}{Test size} & Evaluation metric \\
\midrule
CIFAR10 \citep{krizhevsky2009cifar}    & 10                                & 45000                              & 5000                               & 10000                             & Top-1 accuracy          \\
CIFAR100 \citep{krizhevsky2009cifar}    & 100                               & 45000                              & 5000                               & 10000                             & Top-1 accuracy          \\
ImageNet100 \citep{russakovsky2015imagenet}    & 1000                               & 126689                              & -                               & -                             & -  \\
MIT67 \citep{quattoni2009mit67}                & 67                                & 4690                               & 670                                & 1340                              & Top-1 accuracy          \\
DTD \citep{cimpoi2014dtd}                    & 47                                & 1880                               & 1880                               & 1880                              & Top-1 accuracy          \\
Food \citep{bossard14}                 & 101                               & 68175                              & 7575                               & 25250                             & Top-1 accuracy          \\
SUN397 \citep{xiao2010sun}         & 397                               & 15880                              & 3970                               & 19850                             & Top-1 accuracy          \\
Caltech101 \citep{fei2004learning}            & 101                               & 2525                               & 505                                & 5647                              & Mean per-class accuracy \\
CUB200 \citep{welinder2010caltech}           & 200                               & 4990                               & 1000                               & 5794                              & Mean per-class accuracy \\
Dogs \citep{dataset2011novel, deng2009imagenet} & 120                               & 10800                              & 1200                               & 8580                              & Mean per-class accuracy \\
Flowers \citep{nilsback2008data_flowers102}           & 102                               & 1020                         & 1020                               & 6149                              & Mean per-class accuracy \\
Pets \citep{parkhi2012pets}                  & 37                                & 2940                               & 740                                & 3669                              & Mean per-class accuracy \\
\bottomrule
\end{tabular}
}
\end{center}
\end{table}

\subsection{Details of Architecture and Training}

For training on the CIFAR-10 dataset, we use the modified ResNet-18 encoder \citep{he2016deep, chen2020simple} followed by the 2-layer MLP projector. Specifically, we replace the first convolutional layer with a 3x3 convolution at a stride of 1, removing the initial max pooling operation. For ImageNet-100, we use the ResNet-50 encoder, also followed by the 2-layer MLP projector.

The loss configurations for each dataset are summarized in Table~\ref{table:config}. When training on the CIFAR-10 dataset, we first use the loss in \eqref{eq:loss} where $\alpha$ was fixed at 0.5 and $\tau$ ranging from 0.05 to 1.00. Next, we use the loss in \eqref{eq:loss} where $\tau$ was fixed at $0.1$ and $\alpha$ ranging from 0.0 to 1.0. This entire training process using different loss hyperparameters is repeated for batch sizes of 100, 500, and 2000. 

The models are trained utilizing a single NVIDIA RTX 4090 GPU (for CIFAR-10) or two NVIDIA RTX A5000 GPUs (for ImageNet-100), employing the SGD optimizer with a learning rate of 0.05, a momentum of 0.9, and a weight decay of 1e-4. We apply the cosine learning rate schedule \citep{loshchilov2016sgdr}. Training runs for 1,000 epochs on CIFAR-10 and 200 epochs on ImageNet-100.

\begin{table}[h]
\caption{Training configuration.}
\label{table:config}
\begin{center}
\begin{tabular}{c|ccc}
\toprule
Training dataset  & Batch size      & $\alpha$ (loss-combining coefficient)   & $\tau$ (temperature parameter)          \\
\midrule
CIFAR-10 & 100, 500, 2000 & \begin{tabular}[c]{@{}c@{}}0.5\\ 0.0, 0.1, 0.2, $\cdots$, 1.0\end{tabular} & \begin{tabular}[c]{@{}c@{}}0.05 0.10 0.15, $\cdots$, 1.00\\ 0.1\end{tabular} \\
\cmidrule(lr){1-4}
ImageNet-100      & 256             & 0.0, 0.2, 0.4, 0.5, 0.6, 0.8, 1.0 & 0.1    
\\
\bottomrule
\end{tabular}
\end{center}
\end{table}

\subsection{Details of Linear Probing Evaluation}
\label{sec:appendix:transfer:linear}

To evaluate the average within-class variance of the learned embeddings, we begin by extracting the projector outputs of the training data, without applying any augmentations. These output vectors are subsequently normalized, after which the average within-class variance is computed across all classes.

Subsequently, we remove the projector head and evaluate the performance of the pretrained encoder on various downstream classification tasks by linear probing. Following prior works \citep{kornblith2019better, lee2021improving, oh2024effectiveness}, we evaluate the top-1 accuracy or mean per-class accuracy depending on the downstream dataset, as shown in Table~\ref{table:data}. To be specific, we train the linear classifier by minimizing the L2-regularized cross-entropy loss using limited-memory BFGS \citep{liu1989limited}. The best-performing classifier on the training data is subsequently used to predict the test data, followed by an evaluation of accuracy.

\subsection{Additional Experiments for Evaluating Transfer Learning Performance}
\label{sec:appendix:transfer}

We use the same ResNet-50 encoders from the linear probing evaluations in Table~\ref{table:data}, pretrained on ImageNet-100 with the SupCL loss $\mathcal{L}(\mathbf{U})$ in \eqref{eq:loss}, using $\tau=0.1$ and $\alpha$ ranging from $0$ to $1$. We then evaluate transfer learning performance on object detection and few-shot learning tasks, following related works \citep{he2020momentum, oh2024effectiveness}.

\begin{table}[h]
    \centering
    \caption{Transfer learning performance (\%) on VOC object detection task, using the metric of COCO-style AP on the VOC07 test dataset.}
    \begin{tabular}{c|c}
        \toprule
        $\alpha$ & AP \\
        \midrule
        0.0  & 53.21 \\
        0.2  & \textbf{53.27} \\
        0.5  & 53.09 \\
        0.8  & 51.47 \\
        1.0  & 50.72 \\
        \bottomrule
    \end{tabular}
    \label{table:object:detection}
\end{table}

Table~\ref{table:object:detection} presents the results for the object detection task. We follow the experimental setup of \citet{he2020momentum}, initializing a Faster R-CNN model with a ResNet-50 pre-trained on ImageNet-100 and fine-tuning it on the VOC07+12 training dataset \citep{everingham2010pascal}. Performance is evaluated using the metric of COCO-style average precision (AP) \citep{lin2014microsoft} on the VOC07 test dataset. As shown in Table~\ref{table:object:detection}, the best results are obtained when the loss-combining coefficient $\alpha$ is $0.2$, \ie when the pre-trained embeddings have a moderate amount of within-class variance.

\begin{table}[h]
    \centering
    \caption{Few-shot classification accuracy (\%) evaluated across various downstream datasets.}
    \resizebox{\textwidth}{!}{%
    \begin{tabular}{c|c|ccccc|c|ccccc}
        \toprule
        \multicolumn{1}{c|}{} & \multicolumn{6}{c|}{5-way 1-shot} & \multicolumn{6}{c}{5-way 5-shot} \\
        \cmidrule(lr){2-7} \cmidrule(lr){8-13}
        $\alpha$ & Avg. accuracy & Aircraft & CUB200 & FC100 & Flowers102 & DTD & Avg. accuracy & Aircraft & CUB200 & FC100 & Flowers102 & DTD \\
        \midrule
        0.0  & 51.98  & 31.78  & 49.09  & 43.69  & 78.47  & 56.88 & 68.48  & 45.18  & 66.30  & 62.81  & 93.45  & 74.66 \\
        0.2  & 51.53  & 30.85  & 47.90  & 45.88  & 75.65  & 57.35 & 67.95  & 44.19  & 64.21  & 65.12  & 91.61  & 74.61 \\
        0.5  & \textbf{52.33}  & 31.23  & 49.73  & 46.09  & 76.90  & 57.72 & \textbf{68.99}  & 44.21  & 66.39  & 65.24  & 92.76  & 76.35 \\
        0.8  & 49.84  & 30.42  & 45.08  & 41.03  & 75.00  & 57.66 & 66.54  & 43.40  & 61.11  & 60.60  & 91.52  & 76.07 \\
        1.0  & 46.16  & 29.67  & 40.96  & 35.21  & 69.49  & 55.45 & 61.82  & 39.96  & 55.61  & 51.64  & 88.04  & 73.87 \\
        \bottomrule
    \end{tabular}%
    }
    \label{table:fewshot}
\end{table}

Table~\ref{table:fewshot} shows empirical results for few-shot learning tasks. Following the linear probing protocol for few-shot learning \citep{lee2021improving, oh2024effectiveness}, we first extract representations of 224×224 images (without data augmentation) from the pre-trained model and train a classification head. We then evaluate the accuracy of 5-way 1-shot and 5-way 5-shot scenarios over 2000 episodes across five downstream datasets: Aircraft \citep{maji2013fine}, CUB200 \citep{welinder2010caltech}, FC100 \citep{oreshkin2018tadam}, Flowers \citep{nilsback2008data_flowers102}, and DTD \citep{cimpoi2014dtd}. As shown in Table~\ref{table:fewshot}, the model achieve optimal performance at $\alpha=0.5$, again corresponding to a moderate level of within-class variance.

These additional evaluations demonstrate that embeddings with a moderate amount of within-class variance achieve better performance across diverse transfer learning tasks.

\subsection{Comparisons with Related Methods}

The SupCL loss $\loss(\vU)$ in \eqref{eq:loss}, which we analyze, is formulated as a convex combination of the supervised contrastive loss and the self-supervised contrastive loss. We focus on this loss because it outperforms other existing CL losses \citep{islam2021broad, oh2024effectiveness}. To further demonstrate its effectiveness, we conducted additional experiments comparing the SupCL loss $\loss(\vU)$ in \eqref{eq:loss} with existing CL methods: 

\begin{itemize}
    \item SimCLR \citep{chen2020simple}: A widely adopted self-supervised CL method that does not utilize supervision.
    \item Vanilla SupCL \citep{khosla2020supervised}: The first method that integrates supervision into self-supervised CL.
    \item  $L_\text{spread}$ \citep{chen2022perfectly}: A variant of the SupCL loss designed to mitigate class collapse.
\end{itemize}

These methods are particularly relevant to our work, as our analysis focuses on a loss that is a convex combination of the supervised contrastive loss and the self-supervised contrastive loss. We also included $L_\text{spread}$  \citep{chen2022perfectly} in our comparison, since it was developed to address the class collapse problem.

For the comparison, we trained ResNet-50 on ImageNet-100 using a temperature of 0.1, a learning rate of 0.3, and a batch size of 256. The loss-combining weight $\alpha$ was set to 0.5 for both our method and $L_\text{spread}$. For consistency, the temperature parameter $\tau$ was fixed to 0.1 across all methods.

We evaluate transfer learning performance following Appendix~\ref{sec:appendix:transfer:linear}. As shown in Table~\ref{table:comparison:with:other:methods}, using the SupCL loss $\loss(\vU)$ in \eqref{eq:loss} outperforms other approaches, highlighting its superior effectiveness.

\begin{table}[h] 
    \centering
    \caption{Classification accuracy (\%) evaluated on various downstream datasets.}
    \label{table:comparison:with:other:methods}
    \resizebox{\textwidth}{!}{%
    \begin{tabular}{l | c | c c c c c c c c c c c}
        \toprule
        Method & Avg. accuracy & CIFAR10 & CIFAR100 & Caltech101 & CUB200 & Dog & DTD & Flowers102 & Food101 & MIT67 & Pets & SUN397 \\
        \midrule
        SimCLR & 63.06 & 84.15 & 63.17 & 78.64 & 30.40 & 46.19 & 65.00 & 85.71 & 62.02 & 62.91 & 67.60 & 47.89 \\
        Vanilla SupCL & 68.37 & 88.64 & 69.20 & 87.18 & 35.71 & 62.47 & 66.54 & 88.95 & 58.86 & 63.36 & 80.51 & 50.65 \\
        $L_\text{spread}$ & 68.84 & 89.61 & 69.95 & 87.55 & 37.59 & 62.21 & 65.90 & 89.05 & 60.98 & 63.88 & 79.22 & 51.25 \\
        SupCL in \eqref{eq:loss} & \textbf{69.06} & 89.43 & 69.45 & 88.35 & 38.48 & 62.78 & 66.33 & 89.49 & 60.36 & 63.43 & 80.68 & 50.88 \\
        \bottomrule
    \end{tabular}%
    }
\end{table}

\end{document}